\documentclass[11pt]{article}

\usepackage{authblk}
\usepackage[utf8]{inputenc}
\usepackage[T1]{fontenc}
\usepackage{amsmath,amssymb,amsfonts}
\usepackage{bm}
\usepackage{hyperref}
\usepackage{graphicx}
\usepackage{booktabs}
\usepackage{multirow}
\usepackage{subcaption}
\usepackage{tabularx}
\usepackage{adjustbox}
\usepackage{makecell}
\usepackage{algorithm}
\usepackage{algpseudocode}
\usepackage{geometry}
\usepackage{amsthm}
\newtheorem{theorem}{Theorem}

\newtheorem{proposition}{Proposition}[section]
\theoremstyle{remark}
\newtheorem{remark}{Remark}[section]

\usepackage{natbib}
\title{Adaptive Kernel Selection for Kernelized Diffusion Maps}
\author[1]{Othmane~Aboussaad$^{\ast}$}
\author[2]{Adam~Miraoui$^{\ast}$}
\author[3,4,5]{Boumediene~Hamzi}
\author[3]{Houman~Owhadi}

\affil[1]{CY Cergy Paris University}
\affil[2]{Institut Polytechnique de Paris}
\affil[3]{California Institute of Technology}
\affil[4]{The Alan Turing Institute}
\affil[5]{Imperial College London}
\date{}

\newcommand{\Hk}{\mathcal{H}}
\newcommand{\Id}{\mathrm{Id}}
\newcommand{\R}{\mathbb{R}}
\newcommand{\E}{\mathbb{E}}

\begin{document}
\maketitle
\footnotetext[0]{$^{\ast}$This work was carried out while the author was a visiting student at Imperial College London.}

% ============================================================
% Abstract 
% ============================================================
\begin{abstract}
The selection of an appropriate kernel is a recurring challenge in kernel-based spectral methods.
In \emph{Kernelized Diffusion Maps} (KDM), the kernel determines the accuracy of the RKHS estimator of a diffusion-type operator, hence the quality and stability of the recovered eigenfunctions.
This paper introduces two complementary approaches to adaptive kernel selection for KDM.
First, we develop a \emph{variational outer loop} that learns continuous kernel parameters (bandwidths and mixture weights) by automatic differentiation through the Cholesky-reduced KDM eigenproblem, using a combined eigenvalue-maximization, subspace-orthonormality, and RKHS penalty objective.
Second, we propose an \emph{unsupervised cross-validation pipeline} that selects among multiple kernel families (Gaussian, Mat\'ern, Rational Quadratic, Laplacian) and bandwidths using an eigenvalue-sum score, combined with random Fourier features for scalability.
Both approaches share a common theoretical foundation: we prove Lipschitz dependence of the KDM operators on kernel weights, spectral projector continuity under a gap condition, a residual-control theorem certifying proximity to the target eigenspace, and exponential consistency of the CV selector over a finite kernel dictionary under concentration and margin assumptions.
Experiments on Ornstein--Uhlenbeck processes (2D, 3D, and $d=10$--$20$), double-well potentials, and circle manifolds show that the CV+RFF pipeline achieves near-perfect recovery on OU problems (up to $2\times$ error reduction over uniform baselines) by automatically selecting non-Gaussian kernels. A Mat\'ern-3/2 variational RFF method with bounded anisotropic parameterization further matches or exceeds CV+RFF on the OU2D benchmarks ($0.991$ vs $0.977$ on $\alpha_y=4$; $0.960$ vs $0.939$ on $\alpha_y=16$), demonstrating that CV and variational refinement are complementary when given equal inner-loop capacity.
Beyond positive results, we identify and analyze structural failure modes of unconstrained gradient-based bandwidth optimization---including $\sigma\to 0$ collapse under orthonormality penalties, $\sigma\to\infty$ trivialization under eigenvalue-only losses, and the incapacity of Rayleigh-quadratic regularizers to prevent either---explaining why the two approaches are not symmetric and motivating a hybrid CV-then-local-refinement strategy.
\end{abstract}

% ============================================================
\section{Introduction}
\label{sec:intro}

Spectral methods extract low-dimensional structure from high-dimensional data by computing eigenpairs of operators constructed from pairwise relationships among samples.
Given $N$ data points, one builds a matrix encoding local similarities and analyzes its leading eigenvalues and eigenvectors; the resulting coordinates reveal clusters, manifolds, and slow dynamical modes that are invisible in the ambient representation.
The approach is nonparametric, requires no explicit generative model, and---when the operator is well chosen---provably converges to intrinsic geometric or dynamical quantities as $N\to\infty$.

\emph{Diffusion maps}~\citep{COIFMAN20065} are the most widely used instance of this program.
Starting from a kernel $k_\varepsilon$ and i.i.d.\ samples $\{x_i\}_{i=1}^N$, one constructs a Markov operator whose leading eigenfunctions approximate those of the Laplace--Beltrami operator on the data manifold, or, when the data come from a stochastic differential equation, of the infinitesimal generator.
The resulting diffusion coordinates have become a standard tool for manifold learning, clustering, and the identification of slow reaction coordinates in molecular dynamics and other applications.
Yet the quality of the embedding depends critically on the \emph{choice of kernel}: the kernel family, its bandwidth, and any normalizations jointly determine the operator spectrum and, consequently, the accuracy and stability of the recovered eigenfunctions.

Reproducing kernel Hilbert spaces (RKHS) provide a natural framework for making this dependence precise.
An RKHS $\mathcal{H}_k$ associated with a positive-definite kernel $k$ comes equipped with a norm that encodes smoothness, and operator approximations built in $\mathcal{H}_k$ inherit non-asymptotic error bounds that depend on the regularity of the target eigenfunctions relative to the RKHS rather than on the ambient dimension.
This perspective has been exploited broadly in the analysis of dynamical systems---for learning attractors, Lyapunov functions, Koopman operators, and transfer operators~\citep{CuckerandSmale, BH4, yk1, bhcm11, bhcm1, lyap_bh, bh_sparse_kfs, BHPhysicaD, hamzi2019kernel, bh2020b, SKO3, ALEXANDER2020132520, bh12, bh17, hb17, mmd_kernels_bh, hamzipaillet, jalalian_deds_2023, lee2024note, hamzi2025kernel_lions, bittracher2021dimensionality, bh_kfs_p6}---and motivates a function-space approach to spectral estimation.

\emph{Kernelized Diffusion Maps} (KDM)~\citep{pmlr-v195-pillaud-vivien23a} realize this idea concretely.
Rather than building a graph Laplacian from pointwise kernel evaluations, KDM constructs two RKHS operators---a covariance operator $\Sigma$ and a gradient operator $\mathcal{L}$ assembled from kernel derivatives---and recovers eigenfunctions of the diffusion operator by solving a generalized eigenproblem in $\mathcal{H}_k$.
The resulting estimator achieves dimension-free statistical rates that adapt to the regularity of the target, and admits scalable approximations through Nystr\"om subsampling or random Fourier features~\citep{RR08}.
Complementary results in operator learning confirm that replacing graph Laplacians with function-space estimators improves both statistical and computational performance of spectral algorithms~\citep{pmlr-v238-cabannes24a}, and diffusion-maps kernels have been used as data-adaptive inputs to kernel ridge regression for learning solution operators of dynamical systems~\citep{song2025learningsolutionoperatordynamical}.

Despite these advances, KDM inherits the same kernel-selection problem as classical diffusion maps: a single kernel chosen \emph{a priori} may be poorly matched to the target eigenfunctions, especially under noise or when the eigenfunctions encode \emph{dynamical} rather than purely geometric structure.
A parallel line of work on \emph{generator-based objectives} reinforces this point.
For reversible diffusions, slow modes correspond to eigenfunctions of the infinitesimal generator or, equivalently, the Markov semigroup; learning generator-related quantities from data has been approached through energy-form risks~\citep{NEURIPS2024_f930c6e1}, transform-based methods~\citep{kostic2025laplacetransformbasedlowcomplexity}, and debiasing techniques for data collected under non-equilibrium conditions~\citep{NEURIPS2024_89edef87}.
These results show that physics-informed constraints can guide eigenfunction learning beyond purely geometric criteria---but they do not resolve the kernel-selection question itself.
Multiple kernel learning (MKL)~\citep{gonen2011multiple} offers a principled route: instead of committing to one kernel, one optimizes over convex combinations drawn from a dictionary.
Kernel flows~\citep{owhadi2019kernel} pursue a related idea by selecting kernels that minimize a loss tied to the statistical task at hand, and recent work constructs task-adapted kernels for Koopman eigenfunctions via Lions-type variational principles and multiple-kernel learning~\citep{hamzi2025transport,hamzi2025stochastic}.
Adapting these strategies to the KDM setting---where the task is spectral estimation of a diffusion-type operator---is the subject of this paper.

\paragraph{Our approach.}
We address adaptive kernel selection for KDM through two complementary strategies that share a common operator-theoretic foundation.
The first is a \emph{variational multiple-kernel learning} framework: given a dictionary of positive-definite kernels $\{k_\ell\}_{\ell=1}^L$, we optimize a convex mixture
\begin{equation}
k_\beta(x,y) \;=\; \sum_{\ell=1}^L \beta_\ell\, k_\ell(x,y), \qquad \beta \in \Delta_L,
\label{eq:kernel-mixture}
\end{equation}
by differentiating through the Cholesky-reduced KDM eigenproblem, using an outer objective that blends eigenvalue maximization, RKHS regularization, and optional generator-informed terms.
The kernel $k_\beta$ induces an RKHS $\Hk_{k_\beta}$ and a pair of empirical operators $(\widehat\Sigma_\beta,\widehat{\mathcal{L}}_\beta)$ from which eigenpairs are computed; the variational framework makes this entire pipeline differentiable and computationally explicit.
The second strategy is an \emph{unsupervised cross-validation pipeline} that searches over multiple kernel families (Gaussian, Mat\'ern, Rational Quadratic, Laplacian) and bandwidths using a held-out eigenvalue-sum score, combined with random Fourier features for scalability.
The two approaches share the same theoretical guarantees developed in Section~\ref{sec:theory}.

\paragraph{Contributions.} Our main contributions are as follows.
\begin{itemize}
    \item \textbf{Adaptive kernel selection for KDM.}
    We study two complementary strategies for adapting kernels in kernelized diffusion maps:
    a differentiable variational outer loop for continuous kernel refinement, and a cross-validated
    multi-family kernel-selection pipeline combined with random Fourier features for scalable
    practical model selection.

    \item \textbf{Variational KDM framework.}
    We formulate a variational outer-loop objective for KDM based on eigenvalue-driven spectral
    optimization, subspace-orthonormality stabilization, RKHS regularization, and optional
    generator-informed residual terms. The resulting framework makes end-to-end kernel adaptation
    in KDM differentiable and computationally explicit.

    \item \textbf{Theory.}
    We prove well-posedness of convex kernel mixtures, Lipschitz dependence of the reduced KDM
    operators on kernel weights, spectral projector continuity under a gap condition, a
    residual-control theorem linking generator residuals to proximity to the target eigenspace,
    and exponential consistency of the CV selector over a finite kernel dictionary
    under concentration and margin assumptions.

    \item \textbf{Empirical complementarity.}
    Across synthetic manifolds and stochastic dynamical systems, we show that the CV+RFF pipeline
    achieves the strongest practical performance on OU-type problems by automatically selecting
    advantageous non-Gaussian kernels. A Mat\'ern-3/2 variational RFF with bounded anisotropic
    parameterization then further improves on CV+RFF for the 2D OU benchmarks, demonstrating that
    the two approaches are genuinely complementary when given equal inner-loop capacity.

    \item \textbf{Diagnostic contribution.}
    Beyond positive results, we identify and systematically analyze structural failure modes of
    unconstrained gradient-based bandwidth optimization: $\sigma\to 0$ collapse under
    orthonormality penalties, $\sigma\to\infty$ trivialization under eigenvalue-only losses,
    and the incapacity of Rayleigh-quadratic regularizers to prevent either. We show that these
    pathologies are structural to Rayleigh-based objectives, not tuning issues, which explains
    why CV-based selection is more robust than gradient-based optimization on bandwidth choice.
\end{itemize}

\paragraph{Organization.} The paper is organized as follows:
Section~\ref{sec:background} reviews diffusion maps, KDM, scalable approximations, MKL and diffusion generators.
Section~\ref{sec:method} presents the variational outer loop and the CV+RFF pipeline.
Section~\ref{sec:theory} establishes theoretical properties (well-posedness, stability, residual control, identifiability).
Section~\ref{sec:experiments} reports experiments, ablations and high-dimensional scaling.
Section~\ref{sec:discussion} discusses trade-offs, limitations and future directions.

% ============================================================
\section{Background}
\label{sec:background}

This section reviews diffusion maps, the RKHS-based viewpoint of KDM and scalable approximations, multiple kernel learning (MKL) and infinitesimal generators.

% ------------------------------------------------------------
\subsection{Diffusion maps and kernel integral operators}
\label{subsec:dm}

Let $\mu$ be a probability measure supported on a manifold $\mathcal{M}\subset\mathbb{R}^d$, and let $X=\{x_i\}_{i=1}^N\sim \mu$ be i.i.d.\ samples.
Given a symmetric kernel $k_\varepsilon:\mathcal{M}\times\mathcal{M}\to\mathbb{R}_+$, diffusion maps construct a Markov operator by normalizing a kernel integral operator.
Define the population degree function
\begin{equation}
q_\varepsilon(x) \;=\; \int_{\mathcal{M}} k_\varepsilon(x,y)\, d\mu(y),
\end{equation}
and the $\alpha$-normalized kernel
\begin{equation}
k_\varepsilon^{(\alpha)}(x,y)
\;=\;
\frac{k_\varepsilon(x,y)}{q_\varepsilon(x)^\alpha\, q_\varepsilon(y)^\alpha},
\qquad \alpha\in[0,1].
\label{eq:alpha-normalization}
\end{equation}
Let $d_\varepsilon^{(\alpha)}(x)=\int k_\varepsilon^{(\alpha)}(x,y)\,d\mu(y)$ and define the Markov operator
\begin{equation}
(P_\varepsilon f)(x)
\;=\;
\int_{\mathcal{M}}
\frac{k_\varepsilon^{(\alpha)}(x,y)}{d_\varepsilon^{(\alpha)}(x)}\,
f(y)\, d\mu(y).
\label{eq:markov-operator}
\end{equation}
The diffusion maps embedding is obtained from leading eigenpairs of $P_\varepsilon$ or from a self-adjoint conjugate.
Under suitable regimes where $\varepsilon\to0$, the operator approximates the Laplace--Beltrami operator on $\mathcal{M}$.
A rigorous approach oriented on PDEs, including the case of manifolds with boundary and applications to PDEs on point clouds, is developed in \citet{VAUGHN2024101593}.

% ------------------------------------------------------------
\subsection{Kernelized Diffusion Maps and scalable spectral approximations}
\label{subsec:kdm}
Graph Laplacian estimators rely on local averaging and are statistically fragile in high ambient dimension.
\citet{pmlr-v195-pillaud-vivien23a} propose an RKHS estimator built from two operators: a covariance operator and an RKHS ``Dirichlet/gradient'' operator.
This paves the way to non-asymptotic rates and scalable approximations through Nystr\"om subsampling or random features.

\paragraph{RKHS operators.}
Let $k$ be a $C^2$ positive definite kernel on $\R^d$ with RKHS $\Hk_k$.
Denote by $K_x(\cdot)=k(x,\cdot)\in\Hk_k$ and by $\partial_j K_x(\cdot)=\partial_{x_j}k(x,\cdot)\in\Hk_k$.
Define the covariance operator
\begin{equation}
\Sigma
\;:=\;
\E_{X\sim\mu}\big[ K_X \otimes K_X \big],
\label{eq:pv_sigma_pop}
\end{equation}
and the gradient operator
\begin{equation}
\mathcal{L}
\;:=\;
\E_{X\sim\mu}\Big[\sum_{j=1}^d \partial_j K_X \otimes \partial_j K_X \Big].
\label{eq:pv_L_pop}
\end{equation}
For $\lambda>0$, define the regularized operator
\begin{equation}
\mathcal{L}_\lambda \;:=\; \mathcal{L}+\lambda \Id.
\label{eq:pv_Llam_pop}
\end{equation}
The KDM operator is the self-adjoint compact operator
\begin{equation}
\mathcal{T}_\lambda
\;:=\;
\mathcal{L}_\lambda^{-1/2}\,\Sigma\,\mathcal{L}_\lambda^{-1/2}.
\label{eq:pv_T_pop}
\end{equation}
Eigenpairs of $\mathcal{T}_\lambda$ are equivalently characterized by a generalized eigenproblem:
find $(\mu,\psi)\in\R\times\Hk_k$ with $\psi\neq 0$ such that
\begin{equation}
\Sigma \psi \;=\; \mu\, \mathcal{L}_\lambda \psi.
\label{eq:pv_gen_eig_pop}
\end{equation}
This construction avoids graph local averaging and can circumvent the curse of dimensionality under appropriate regularity \citep{pmlr-v195-pillaud-vivien23a}.

\paragraph{Empirical estimators.}
Given samples $X=\{x_i\}_{i=1}^N\sim\mu$, define
\begin{equation}
\widehat{\Sigma}
\;:=\;
\frac1N\sum_{i=1}^N K_{x_i}\otimes K_{x_i},
\qquad
\widehat{\mathcal{L}}
\;:=\;
\frac1N\sum_{i=1}^N\sum_{j=1}^d \partial_j K_{x_i}\otimes \partial_j K_{x_i},
\qquad
\widehat{\mathcal{L}}_\lambda=\widehat{\mathcal{L}}+\lambda \Id,
\label{eq:pv_emp_ops}
\end{equation}
and the empirical KDM operator
\begin{equation}
\widehat{\mathcal{T}}_\lambda
\;:=\;
\widehat{\mathcal{L}}_\lambda^{-1/2}\,\widehat{\Sigma}\,\widehat{\mathcal{L}}_\lambda^{-1/2}.
\label{eq:pv_emp_T}
\end{equation}
In practice, one computes leading eigenpairs through a finite-dimensional approximation by using either Nystr\"om subsampling or random features \citep{pmlr-v195-pillaud-vivien23a}.

\paragraph{Nystr\"om viewpoint.}
KDM  is \emph{not} the eigen-decomposition of a covariance surrogate $W^{-1/2}C^\top C W^{-1/2}$ alone.
Instead, the finite-dimensional inner loop builds matrices that approximate the pair $(\widehat{\Sigma},\widehat{\mathcal{L}})$ (the latter requiring kernel derivatives) and solves a generalized eigenproblem that discretizes \eqref{eq:pv_gen_eig_pop}.
We detail the Nystr\"om construction in Section~\ref{subsec:nystrom} and Appendix~\ref{app:spectral}.

\paragraph{Galerkin viewpoint.}
Beyond Nystr\"om, restricting the spectral problem to a low-dimensional test-function space (Galerkin) can be statistically and computationally advantageous compared to graph-based approaches \citep{pmlr-v238-cabannes24a}.

% ------------------------------------------------------------
\subsection{Multiple kernel learning (MKL) and RKHS of a convex mixture}
\label{subsec:mkl}

We consider a dictionary of positive definite kernels $\{k_\ell\}_{\ell=1}^L$ and a convex mixture
\begin{equation}
k_\beta(x,y)\;=\;\sum_{\ell=1}^L \beta_\ell\, k_\ell(x,y),
\qquad
\beta\in\Delta_L
:=\Big\{\beta\in\mathbb{R}^L_{\ge 0}:\sum_{\ell=1}^L \beta_\ell=1\Big\}.
\label{eq:kernel-mixture-bg}
\end{equation}
Convex combinations preserve positive definiteness, hence $k_\beta$ defines an RKHS $\Hk_{k_\beta}$.
A useful characterization of the RKHS norm is the infimal convolution form:
\begin{equation}
\|f\|^2_{\mathcal{H}_{k_\beta}}
\;=\;
\inf_{\substack{f=\sum_{\ell=1}^L f_\ell\\ f_\ell\in \mathcal{H}_{k_\ell}}}
\ \sum_{\ell=1}^L \frac{\|f_\ell\|^2_{\mathcal{H}_{k_\ell}}}{\beta_\ell}.
\label{eq:rkhs-mixture-norm}
\end{equation}
This identity makes explicit how $\beta$ controls the geometry of functions and motivates RKHS penalties when learning $\beta$.

From a practical standpoint, the importance of learning kernels in dynamical settings is also highlighted by the \emph{kernel flows} criterion of~\citet{owhadi2019kernel}, and its application to learning kernels for dynamical systems in a sequence of papers including~\citet{hamzi2021learning,darcy2025learning}.

% ------------------------------------------------------------
\subsection{Infinitesimal generators, semigroups and eigenfunctions}
\label{subsec:generators}

Consider a diffusion process $(X_t)_{t\ge 0}$ on $\mathcal{M}$ satisfying
\begin{equation}
dX_t = b(X_t)\,dt + \sigma(X_t)\, dW_t,
\end{equation}
with drift $b$ and diffusion matrix $a=\sigma\sigma^\top$.
The infinitesimal generator $\mathcal{G}$ acts on smooth test functions as
\begin{equation}
(\mathcal{G}f)(x)
\;=\;
b(x)^\top \nabla f(x)
+\frac12 \mathrm{tr}\!\big(a(x)\nabla^2 f(x)\big),
\label{eq:generator}
\end{equation}
and generates the Markov semigroup $(P_t)_{t\ge 0}$ via $P_t=e^{t\mathcal{G}}$.
Slow collective variables and reaction coordinates are linked to leading eigenfunctions of $\mathcal{G}$ (or of $P_t$), motivating objectives that encourage eigen-relations
\begin{equation}
\mathcal{G}\phi \;=\; -\lambda \phi,
\label{eq:gen-eigen}
\end{equation}
either pointwise via a residual or in an energy form.
We reserve the symbol $\mathcal{L}$ for the RKHS Dirichlet operator in \eqref{eq:pv_L_pop} to avoid notational clashes.

Another recent work provides learning objectives and guarantees for estimating quantities related to the generator from data. Indeed, \cite{NEURIPS2024_f930c6e1} develops a learning framework for $\mathcal{L}$ using Dirichlet form risks with finite-sample guarantees, while \cite{kostic2025laplacetransformbasedlowcomplexity} proposes low-complexity learning of continuous Markov semigroups via Laplace transform. Finally, diffusion-map-based constructions can be leveraged in dynamical and generative contexts, therefore illustrating the broader relevance of diffusion-operator spectra \citep{li2023diffusion}.

% ============================================================
\section{Method}
\label{sec:method}

We present two complementary approaches to adaptive kernel selection for KDM.
The first is a \emph{variational multiple-kernel learning} (VMKL) method in which kernel
parameters are learned by differentiating through a Nystr\"om approximation of the KDM
generalized eigenproblem. The second, used primarily in the experimental section, is a
\emph{cross-validated multi-family kernel-selection pipeline} combined with random Fourier
features (RFF). The variational method provides the differentiable formulation and theoretical
backbone of continuous kernel adaptation, while the CV+RFF pipeline provides a broader and more
scalable practical search over kernel families and bandwidths. We describe the variational method
in detail below and use the CV+RFF pipeline in Section~\ref{sec:experiments} as a complementary adaptive baseline.

% ------------------------------------------------------------
\subsection{Problem setup and notation}
\label{subsec:setup}

Let $X=\{x_i\}_{i=1}^N\subset\mathbb{R}^d$ denote the dataset and let $Z=\{z_j\}_{j=1}^p$ denote Nystr\"om landmarks with $p\ll N$.
We consider a dictionary of $C^2$ positive definite kernels $\{k_\ell\}_{\ell=1}^L$ and learn mixture weights $\beta\in\Delta_L$.
For a given $\beta$, the mixture kernel is
\begin{equation}
k_\beta(x,y) \;=\; \sum_{\ell=1}^L \beta_\ell\, k_\ell(x,y).
\label{eq:kbeta}
\end{equation}

\paragraph{Target eigenfunctions.}
For each $\beta$ and regularization $\lambda>0$, KDMs define the empirical operator
\[
\widehat{\mathcal{T}}_{\beta,\lambda}
=
\widehat{\mathcal{L}}_{\beta,\lambda}^{-1/2}\,\widehat{\Sigma}_\beta\,\widehat{\mathcal{L}}_{\beta,\lambda}^{-1/2},
\qquad
\widehat{\mathcal{L}}_{\beta,\lambda}=\widehat{\mathcal{L}}_\beta+\lambda \Id,
\]
whose leading eigenfunctions yield the KDM embedding.
Throughout, $\lambda$ denotes the RKHS regularization parameter, $\mu_k$ denotes the generalized eigenvalues of the KDM eigenproblem, and $\widehat{\lambda}_k$ denotes Rayleigh-quotient estimates of the generator eigenvalues used in the optional PDE residual term.
We compute approximate eigenfunctions $\{\widehat{\phi}_k(\beta)\}_{k=1}^r$ on samples via Nystr\"om (Section~\ref{subsec:nystrom}).

\paragraph{Dynamical information.}
When a generator $\mathcal{G}$ (or a discrete approximation $G_h$) is available, we may include an outer-loop loss that encourages $\widehat{\phi}_k$ to satisfy $\mathcal{G}\widehat{\phi}_k\approx -\widehat{\lambda}_k\widehat{\phi}_k$ in $L^2(\mu)$.
This term is removed when treating purely geometric settings.

% ------------------------------------------------------------
\subsection{Rayleigh-Ritz interpretation of the KDM inner loop}
\label{subsec:rayleigh-ritz}

The KDM generalized eigenproblem admits a classical variational interpretation that clarifies the role of each component.
For fixed $\beta$ and $\lambda$, the population operator $\mathcal{T}_{\beta,\lambda}$ has eigenfunctions characterized by the min-max Rayleigh principle:
\begin{equation}
\label{eq:rayleigh-minmax}
\mu_j
=
\max_{\substack{V\subset\mathcal{H}_{k_\beta}\\ \dim V = j}}
\min_{\phi\in V\setminus\{0\}}
\frac{\langle \phi, \Sigma_\beta\, \phi\rangle}{\langle \phi, \mathcal{L}_{\beta,\lambda}\,\phi\rangle}.
\end{equation}
The empirical KDM inner loop is precisely the \emph{Rayleigh-Ritz} (Galerkin) approximation of~\eqref{eq:rayleigh-minmax} in the finite-dimensional trial subspace $V_N = \mathrm{span}\{k_\beta(\cdot,x_1),\ldots,k_\beta(\cdot,x_N)\}$ (or $V_p = \mathrm{span}\{k_\beta(\cdot,z_j)\}_{j=1}^p$ in the Nystr\"om case).
By standard Rayleigh-Ritz theory, when $\mathcal{L}_{\beta,\lambda}$ is positive definite, the Ritz values $\widehat\mu_1\ge\cdots\ge\widehat\mu_r$ approximate the corresponding population eigenvalues from above (interlacing), and converge as the trial subspace fills $\mathcal{H}_{k_\beta}$.
This connection has three practical consequences:
(i) the eigenvalue-sum CV score~\eqref{eq:emp-cv} is exactly the sum of top-$r$ Rayleigh quotients evaluated at the Ritz vectors, providing the connection to the Rayleigh CV score of Section~\ref{subsec:rayleigh-cv};
(ii) the orthogonality constraint in our variational outer loop is the Galerkin orthogonality condition, not an ad-hoc penalty; 
(iii) the Lipschitz bounds of Proposition~\ref{prop:lipschitz-operators} translate, via standard Rayleigh-Ritz perturbation theory, into stability bounds for the eigenvalues with respect to $\beta$, underpinning the existence of a minimizer in Proposition~\ref{prop:existence-minimizer}.
Throughout the paper, we use ``Rayleigh quotient'' and ``generalized eigenvalue'' interchangeably when referring to the top-$r$ objective.

% ------------------------------------------------------------
\subsection{Mixture-of-kernels parametrization (VMKL)}
\label{subsec:vmkl}

We parameterize the mixture weights with a softmax map to enforce simplex constraints:
\begin{equation}
\beta_\ell(u) \;=\; \frac{\exp(u_\ell)}{\sum_{m=1}^L \exp(u_m)}, \qquad \ell=1,\dots,L,
\label{eq:softmax}
\end{equation}
where $u\in\mathbb{R}^L$ are unconstrained parameters.
We can also prevent collapse by a small floor $\beta \leftarrow (1-\tau)\beta + \tau\mathbf{1}/L$.

% ------------------------------------------------------------
\subsection{Inner loop: Nystr\"om approximation of KDM}
\label{subsec:nystrom}

For each base kernel $k_\ell$, define Nystr\"om kernel matrices
\begin{equation}
(C_\ell)_{ij} = k_\ell(x_i,z_j), \qquad (W_\ell)_{ij}=k_\ell(z_i,z_j),
\label{eq:CW-ell}
\end{equation}
with $C_\ell\in\mathbb{R}^{N\times p}$, $W_\ell\in\mathbb{R}^{p\times p}$.

\paragraph{Kernel-derivative Nystr\"om matrices.}
KDM requires kernel derivatives.
For $k_\ell\in C^2$ define the derivative cross-matrix $J_\ell\in\mathbb{R}^{(Nd)\times p}$ by
\begin{equation}
(J_\ell)_{(i,j),m} \;=\; \partial_{x_j} k_\ell(x_i,z_m),
\qquad i=1,\dots,N,\; j=1,\dots,d,\; m=1,\dots,p,
\label{eq:J-ell}
\end{equation}
where $(i,j)$ indexes the flattened pair.
(For $d=1$, this is simply an $N\times p$ matrix of first derivatives.)

\paragraph{Mixture aggregation.}
Given mixture weights $\beta$, we form aggregated Nystr\"om matrices
\begin{equation}
C_\beta = \sum_{\ell=1}^L \beta_\ell C_\ell,
\qquad
W_\beta = \sum_{\ell=1}^L \beta_\ell W_\ell,
\qquad
J_\beta = \sum_{\ell=1}^L \beta_\ell J_\ell.
\label{eq:CWJ-beta}
\end{equation}
We stabilize $W_\beta$ by symmetrization and jitter:
\begin{equation}
W_\beta \leftarrow \frac{W_\beta+W_\beta^\top}{2} + \varepsilon_W I_p.
\label{eq:W-stab}
\end{equation}

\paragraph{Finite-dimensional operators.}
In the landmark span $\mathrm{span}\{k_\beta(\cdot,z_m)\}_{m=1}^p$, the empirical covariance and Dirichlet operators admit the matrix surrogates
\begin{equation}
\widehat{\Sigma}_\beta^{(p)}
\;:=\;
\frac{1}{N}\,C_\beta^\top C_\beta
\in\mathbb{R}^{p\times p},
\qquad
\widehat{L}_\beta^{(p)}
\;:=\;
\frac{1}{N}\,J_\beta^\top J_\beta
\in\mathbb{R}^{p\times p}.
\label{eq:SigmaL_mats}
\end{equation}
Regularization $\lambda>0$ in RKHS corresponds to adding $\lambda\Id$ in operator form, which becomes $\lambda W_\beta$ in landmark coordinates.
We define the regularized matrix
\begin{equation}
\widehat{L}_{\beta,\lambda}^{(p)}
\;:=\;
\widehat{L}_\beta^{(p)} + \lambda W_\beta.
\label{eq:Llam_mat}
\end{equation}

\paragraph{Generalized eigenproblem.}
We compute top-$r$ eigenpairs $(\mu_k, a_k)$ by solving
\begin{equation}
\widehat{\Sigma}_\beta^{(p)}\, a_k
\;=\;
\mu_k\,
\widehat{L}_{\beta,\lambda}^{(p)}\, a_k,
\qquad
a_k\in\mathbb{R}^p,
\label{eq:gen_eig_mat}
\end{equation}
with eigenvalues ordered $\mu_1\ge \mu_2\ge \cdots$.
Equivalently, $\mu_k$ are the leading generalized Rayleigh quotients associated with the matrix pair $(\widehat{\Sigma}^{(p)}_\beta,\widehat{L}^{(p)}_{\beta,\lambda})$.
This discretizes the characterization $\widehat{\Sigma}_\beta \psi = \mu \widehat{\mathcal{L}}_{\beta,\lambda}\psi$.

\paragraph{Lifting to sample evaluations.}
Given coefficients $a_k$, we evaluate the corresponding eigenfunction on samples by
\begin{equation}
\widehat{\phi}_k
\;=\;
C_\beta a_k
\in\mathbb{R}^N.
\label{eq:lifting_pv}
\end{equation}
Depending on the application, we may remove the constant-like mode and enforce centering and orthogonality constraints empirically.

\paragraph{Normalization and orthogonality.}
We enforce empirical $L^2(\mu)$ constraints to stabilize the outer-loop optimization:
\begin{equation}
\frac{1}{N}\mathbf{1}^\top \widehat{\phi}_k = 0,
\qquad
\frac{1}{N}\|\widehat{\phi}_k\|_2^2 = 1,
\qquad
\frac{1}{N}\widehat{\phi}_k^\top \widehat{\phi}_j = 0 \ (k\neq j),
\label{eq:normalize-orth}
\end{equation}
implemented by centering and Gram--Schmidt (or QR) orthonormalization in $\mathbb{R}^N$.
This gauge-fixing step resolves sign ambiguities (especially near eigenvalue crossings) and stabilizes
differentiation through the eigensolver in the outer loop. It does not change the estimated eigenspace.

% ------------------------------------------------------------
\subsection{Outer loop: variational objective}
\label{subsec:outer}

We optimize $u$ (hence $\beta$) by minimizing a practical outer objective that combines
spectral, subspace, RKHS, and optional generator-informed terms. This formulation is chosen to
match the losses used in the ablations of Section~\ref{sec:experiments}.

\paragraph{(a) Spectral term.}
Let $\mu_1(\beta)\ge \cdots \ge \mu_r(\beta)$ denote the leading generalized eigenvalues from
\eqref{eq:gen_eig_mat}. To encourage a strong leading KDM subspace, we use
\begin{equation}
\mathcal{L}_{\mathrm{eig}}(\beta) := - \sum_{k=1}^r \mu_k(\beta).
\label{eq:eig-loss}
\end{equation}
This term is especially effective on OU-type problems, but by itself can favor overly smooth kernels.

\paragraph{(b) Subspace orthonormality / anchoring.}
To stabilize optimization and prevent collapse, we enforce centering, normalization, and empirical
orthogonality of the lifted eigenfunctions:
\begin{equation}
\mathcal{L}_{\mathrm{sub}}(\beta)
:=
\sum_{k=1}^r
\Big(\tfrac{1}{N}\mathbf{1}^\top \widehat{\phi}_k\Big)^2
+
\sum_{k=1}^r
\Big(\tfrac{1}{N}\|\widehat{\phi}_k\|_2^2-1\Big)^2
+
\eta\!\!\sum_{1\le k<j\le r}
\Big(\tfrac{1}{N}\widehat{\phi}_k^\top \widehat{\phi}_j\Big)^2.
\label{eq:anchor-loss}
\end{equation}
This term fixes the empirical gauge and stabilizes differentiation through the eigensolver.

\paragraph{(c) RKHS regularization.}
In landmark coordinates $\widehat{\phi}_k(\cdot)=\sum_{m=1}^p (a_k)_m\, k_\beta(\cdot,z_m)$, the RKHS norm is
$\|\widehat{\phi}_k\|^2_{\Hk_{k_\beta}} = a_k^\top W_\beta a_k$.
We define
\begin{equation}
\mathcal{L}_{\mathrm{RKHS}}^{(k)}(\beta):=a_k^\top W_\beta a_k.
\label{eq:rkhs-loss}
\end{equation}

\paragraph{(d) Optional generator / PDE residual.}
When a generator discretization $G_h$ is available, we include
\begin{equation}
\mathcal{L}_{\mathrm{PDE}}^{(k)}(\beta)
=
\frac{1}{N}\sum_{i=1}^N
\Big( (G_h \widehat{\phi}_k)_i + \widehat{\lambda}_k\, (\widehat{\phi}_k)_i \Big)^2,
\label{eq:pde-loss}
\end{equation}
where $\widehat{\lambda}_k$ is the Rayleigh estimate from Appendix~\ref{app:pde:rayleigh}.
This term is omitted in purely geometric experiments.

\paragraph{Total objective.}
Combining the above, we minimize
\begin{equation}
\min_{u\in\mathbb{R}^L}
\quad
\mathcal{J}(u)
:=
\tau\,\mathcal{L}_{\mathrm{eig}}(\beta(u))
+
\alpha\,\mathcal{L}_{\mathrm{sub}}(\beta(u))
+
\gamma \sum_{k=1}^r \mathcal{L}_{\mathrm{RKHS}}^{(k)}(\beta(u))
+
\zeta \sum_{k=1}^r \mathcal{L}_{\mathrm{PDE}}^{(k)}(\beta(u))
+
\rho\,\Omega(\beta(u)),
\label{eq:total-objective}
\end{equation}
where $\tau,\alpha,\gamma,\zeta,\rho\ge 0$ are hyperparameters and $\Omega(\beta)$ is an
optional regularizer on the kernel weights.
This formulation subsumes the ablations used in Section~\ref{sec:experiments}:
\textbf{SubOnly} sets $\tau=0,\alpha>0$;
\textbf{EigOnly} sets $\tau>0,\alpha=0$;
\textbf{Combined} sets $\tau>0,\alpha>0,\gamma>0$, with $\zeta>0$ when generator information is available.

% ------------------------------------------------------------
\subsection{Joint optimization and computational complexity}
\label{subsec:opt}

We optimize \eqref{eq:total-objective} end-to-end by differentiating through the generalized eigen-solver whenever feasible.
We use Adam for warm-starting and optionally switch to L-BFGS for refinement.
To mitigate instabilities near eigenvalue crossings, we employ anchoring \eqref{eq:anchor-loss} and, when needed, losses on subspaces.

\paragraph{Precomputation.}
We precompute $\{C_\ell,W_\ell,J_\ell\}_{\ell=1}^L$ once.
The dominant cost is $O(LNp + Lp^2 + L N p d)$ due to derivative matrices.

\paragraph{Per-iteration cost.}
Forming $C_\beta,W_\beta,J_\beta$ costs $O(LNp + Lp^2 + L N p d)$.
Building $\widehat{\Sigma}_\beta^{(p)}$ and $\widehat{L}_\beta^{(p)}$ costs $O(Np^2 + Nd\,p^2)$ if done naively, but can be implemented as matrix products:
$C_\beta^\top C_\beta$ and $J_\beta^\top J_\beta$.
Solving the $p\times p$ generalized eigenproblem \eqref{eq:gen_eig_mat} costs $O(p^3)$ worst-case or $O(rp^2)$ for partial solvers.
Evaluating outer losses costs $O(Nr)$ plus the cost of applying $G_h$ (if PDE term enabled).

% ------------------------------------------------------------
\begin{algorithm}[t]
\caption{Variational Multiple Kernel Learning for Kernelized Diffusion Maps (VMKL--KDM)}
\label{alg:vmkl-kdm}
\begin{algorithmic}[1]
\Require Data $X=\{x_i\}_{i=1}^N$, landmarks $Z=\{z_j\}_{j=1}^p$, kernels $\{k_\ell\}_{\ell=1}^L$, modes $r$,
regularization $\lambda>0$, hyperparameters $(\gamma,\alpha,\rho,\eta,\varepsilon_W)$, generator discretization $G_h$.
\Ensure Mixture weights $\beta^\star$, eigenfunctions $\{\widehat{\phi}_k\}_{k=1}^r$.
\State Precompute $C_\ell,W_\ell,J_\ell$ via \eqref{eq:CW-ell} and \eqref{eq:J-ell} for all $\ell$.
\State Initialize $u\in\mathbb{R}^L$ (e.g., $u=0$ so $\beta$ is uniform).
\For{$t=1,\dots,T$}
  \State Compute $\beta=\mathrm{softmax}(u)$.
  \State Form $C_\beta,W_\beta,J_\beta$ via \eqref{eq:CWJ-beta}; stabilize $W_\beta$ via \eqref{eq:W-stab}.
  \State Build $\widehat{\Sigma}_\beta^{(p)}=\frac1N C_\beta^\top C_\beta$ and $\widehat{L}_\beta^{(p)}=\frac1N J_\beta^\top J_\beta$.
  \State Form $\widehat{L}_{\beta,\lambda}^{(p)}=\widehat{L}_\beta^{(p)}+\lambda W_\beta$.
  \State Solve generalized eigenproblem $\widehat{\Sigma}_\beta^{(p)} a_k=\mu_k \widehat{L}_{\beta,\lambda}^{(p)} a_k$ for top $r$ eigenpairs.
  \State Lift eigenfunctions $\widehat{\phi}_k = C_\beta a_k$ and normalize via \eqref{eq:normalize-orth}.
  \State Compute $\mathcal{J}(u)$ via \eqref{eq:total-objective} (omit $\mathcal{L}_{\mathrm{PDE}}$ if no $G_h$).
  \State Update $u \leftarrow u - \eta_t \nabla_u \mathcal{J}(u)$ (Adam/L-BFGS).
\EndFor
\State \Return $\beta^\star=\mathrm{softmax}(u)$ and $\{\widehat{\phi}_k\}_{k=1}^r$.
\end{algorithmic}
\end{algorithm}

% ============================================================
\section{Theory and Properties}
\label{sec:theory}

We establish rigorous properties of the adaptive-kernel KDM framework: well-posedness of the kernel mixture, Lipschitz dependence of the empirical operators on the mixture weights, spectral stability under a gap condition, existence of an optimizer, and the precise sense in which residual control certifies proximity to the target eigenspace.

\paragraph{Standing assumptions.}
Throughout Section~\ref{sec:theory}, we assume: (i) each base kernel $k_\ell$ is positive definite and $C^2$; (ii) the landmarks $Z=\{z_j\}_{j=1}^p$ are fixed (sensitivity to $Z$ is addressed empirically in Section~\ref{sec:experiments}); (iii) the regularization $\lambda>0$ is bounded away from zero; (iv) when spectral-gap conditions are invoked, we require a strictly positive gap $\delta>0$ at rank $r$ uniformly on the relevant subset of $\Delta_L$.
The theorem-by-theorem dependence is:
Prop.~\ref{prop:pd-mixture} and Prop.~\ref{prop:rkhs-mixture-norm} require only (i);
Prop.~\ref{prop:lipschitz-operators} requires (i) with bounded operator norms $M_C, M_J, M_W$;
Prop.~\ref{prop:uniform-pd} requires (i), (iii), and a uniform positive-definiteness condition on $W_\beta$;
Thm.~\ref{thm:projector-continuity} requires (i)--(iv);
Prop.~\ref{prop:existence-minimizer} requires continuity of $F$ and (iii);
Thm.~\ref{thm:residual-gap} requires self-adjointness of $\mathcal{G}$ with positive gap $\gamma_m$;
Props.~\ref{prop:sign-fixing}--\ref{prop:procrustes} are purely algebraic.

% ------------------------------------------------------------
\subsection{Well-posedness of convex kernel mixtures and RKHS geometry}
\label{subsec:theory-mixture}

Let $\{k_\ell\}_{\ell=1}^L$ be positive definite kernels on $\mathcal{X}$ and $k_\beta=\sum_{\ell=1}^L \beta_\ell k_\ell$ for $\beta\in\Delta_L$.

\begin{proposition}[Positive definiteness of mixtures]
\label{prop:pd-mixture}
If each $k_\ell$ is positive definite and $\beta\in\Delta_L$, then $k_\beta$ is positive definite.
\end{proposition}

\begin{proof}
For any $\{x_i\}_{i=1}^n\subset\mathcal{X}$ and $c\in\mathbb{R}^n$,
$c^\top K_\beta c = \sum_{\ell=1}^L \beta_\ell\, c^\top K_\ell c \ge 0$
since each $K_\ell$ is PSD and $\beta_\ell\ge 0$.
\end{proof}

\begin{proposition}[RKHS norm of a convex mixture]
\label{prop:rkhs-mixture-norm}
For $f\in \Hk_{k_\beta}$,
\[
\|f\|^2_{\Hk_{k_\beta}}
=
\inf_{\substack{f=\sum_{\ell} f_\ell\\ f_\ell\in\Hk_{k_\ell}}}
\sum_{\ell=1}^L \frac{\|f_\ell\|^2_{\Hk_{k_\ell}}}{\beta_\ell},
\]
with the convention that if $\beta_\ell=0$ then $f_\ell=0$.
\end{proposition}

\begin{proof}
The Hilbert direct sum $\mathcal{H}_\oplus=\bigoplus_\ell \Hk_{k_\ell}$ with norm $\|(f_\ell)\|^2_{\oplus,\beta}=\sum_\ell \beta_\ell^{-1}\|f_\ell\|^2_{\Hk_{k_\ell}}$ and the summation map $S(f_1,\dots,f_L)=\sum_\ell f_\ell$ yield an RKHS with kernel $k_\beta$ and the stated quotient norm.
\end{proof}

% ------------------------------------------------------------
\subsection{Lipschitz dependence of empirical operators on kernel weights}
\label{subsec:theory-lipschitz}

\begin{proposition}[Lipschitz continuity of Nystr\"om operators]
\label{prop:lipschitz-operators}
Let $M_C=\max_\ell\|C_\ell\|_2$, $M_J=\max_\ell\|J_\ell\|_2$, $M_W=\max_\ell\|W_\ell\|_2$.
Then for any $\beta,\beta'\in\Delta_L$,
\begin{align}
\|\widehat{\Sigma}^{(p)}_\beta - \widehat{\Sigma}^{(p)}_{\beta'}\|_2 &\le \frac{2M_C^2}{N}\|\beta-\beta'\|_1, \label{eq:lip-sigma}\\
\|\widehat{L}^{(p)}_{\beta,\lambda} - \widehat{L}^{(p)}_{\beta',\lambda}\|_2 &\le \left(\frac{2M_J^2}{N}+\lambda M_W\right)\|\beta-\beta'\|_1. \label{eq:lip-L}
\end{align}
\end{proposition}

\begin{proof}
Since $C_\beta-C_{\beta'}=\sum_\ell(\beta_\ell-\beta'_\ell)C_\ell$, the triangle inequality gives $\|C_\beta-C_{\beta'}\|_2\le M_C\|\beta-\beta'\|_1$.
Factoring $\widehat{\Sigma}^{(p)}_\beta-\widehat{\Sigma}^{(p)}_{\beta'}=N^{-1}(C_\beta^\top(C_\beta-C_{\beta'})+(C_\beta-C_{\beta'})^\top C_{\beta'})$ and using $\|C_\beta\|_2\le M_C$ yields~\eqref{eq:lip-sigma}. The bound~\eqref{eq:lip-L} follows analogously.
\end{proof}

\begin{proposition}[Uniform positive definiteness]
\label{prop:uniform-pd}
If $\lambda>0$ and $W_\beta\succeq\gamma I_p$ for all $\beta\in\Delta_L$, then $\widehat{L}^{(p)}_{\beta,\lambda}\succeq\lambda\gamma I_p$ for all $\beta$, and the generalized eigenproblem is well-posed everywhere on $\Delta_L$.
\end{proposition}

% ------------------------------------------------------------
\subsection{Spectral stability under a gap condition}
\label{subsec:theory-stability}

\begin{theorem}[Continuity of spectral projectors]
\label{thm:projector-continuity}
Let $A_\beta=(\widehat{L}^{(p)}_{\beta,\lambda})^{-1/2}\,\widehat{\Sigma}^{(p)}_\beta\,(\widehat{L}^{(p)}_{\beta,\lambda})^{-1/2}$.
If the spectral gap at rank $r$ satisfies $\widehat\mu_r(A_\beta)-\widehat\mu_{r+1}(A_\beta)\ge\delta>0$ uniformly on $U\subset\Delta_L$, then the spectral projector $P_\beta$ onto the leading $r$-dimensional eigenspace is Lipschitz:
\[
\|P_\beta-P_{\beta'}\|_2 \le \frac{2C_0}{\delta}\|\beta-\beta'\|_1,\qquad \beta,\beta'\in U,
\]
where $C_0$ is the Lipschitz constant of $\beta\mapsto A_\beta$ from Proposition~\ref{prop:lipschitz-operators}.
\end{theorem}

\begin{proof}
By Propositions~\ref{prop:lipschitz-operators} and~\ref{prop:uniform-pd}, $\beta\mapsto A_\beta$ is Lipschitz.
Applying the Davis--Kahan $\sin\Theta$ theorem under the uniform gap $\delta$ yields $\|P_\beta-P_{\beta'}\|_2\le (2/\delta)\|A_\beta-A_{\beta'}\|_2$.
\end{proof}

% ------------------------------------------------------------
\subsection{Existence of a minimizer and residual control}
\label{subsec:theory-existence}

\begin{proposition}[Existence of a minimizer]
\label{prop:existence-minimizer}
If $\widehat{L}^{(p)}_{\beta,\lambda}$ is positive definite for all $\beta\in\Delta_L$ and the outer objective $F(\beta)$ is continuous, then $F$ attains a global minimum on $\Delta_L$.
\end{proposition}

\begin{proof}
$\Delta_L$ is compact and $F$ is continuous; apply the Weierstrass theorem.
\end{proof}

\begin{theorem}[Residual control implies proximity to target eigenspace]
\label{thm:residual-gap}
Let $\mathcal{G}$ be a self-adjoint generator with eigenvalues $\{-\lambda_j\}$ and spectral gap $\gamma_m=\inf_{j\ne m}|\lambda_j-\lambda_m|>0$ at mode $m$.
If $\phi\in L^2(\mu)$ with $\|\phi\|=1$ and residual $r=\mathcal{G}\phi+\lambda_m\phi$, then
\[
\mathrm{dist}(\phi,\mathrm{span}\{e_m\}) \le \frac{\|r\|}{\gamma_m}.
\]
More generally, for a cluster $I$ with gap $\gamma_I=\inf_{j\notin I,i\in I}|\lambda_j-\lambda_i|>0$,
\[
\mathrm{dist}(\phi,\mathcal{E}_I) \le \frac{\|(\mathcal{G}+\lambda)\phi\|}{\gamma_I}.
\]
\end{theorem}

\begin{proof}
Expanding $\phi=\sum_j c_j e_j$ gives $\|r\|^2=\sum_{j\ne m}|\lambda_m-\lambda_j|^2|c_j|^2\ge\gamma_m^2\sum_{j\ne m}|c_j|^2=\gamma_m^2\,\mathrm{dist}(\phi,\mathrm{span}\{e_m\})^2$.
The cluster case follows by projecting onto $\mathcal{E}_I^\perp$ where $\mathcal{G}+\lambda$ has spectrum bounded below by $\gamma_I$.
\end{proof}

\paragraph{Significance.}
Theorem~\ref{thm:residual-gap} makes the generator residual term $\mathcal{L}_{\mathrm{PDE}}$ mathematically meaningful: small residual is not merely a heuristic objective but a certificate of proximity to the target eigenspace whenever a spectral gap is available.
Concretely, the empirical PDE loss $\mathcal{L}_{\mathrm{PDE}}=\frac{1}{N}\|G_h\hat\phi_k+\hat\lambda_k\hat\phi_k\|_2^2$ in~\eqref{eq:total-objective} is a finite-sample estimate of $\|r\|^2$; Theorem~\ref{thm:residual-gap} then guarantees that minimizing this loss drives $\hat\phi_k$ toward the true eigenspace at a rate controlled by the inverse spectral gap $1/\gamma_m$.
In practice, when the KDM estimator is consistent for the generator spectrum, $\gamma_m$ can be estimated from the CV-selected eigenvalues $\hat\mu_1>\cdots>\hat\mu_r>\hat\mu_{r+1}$ as $\hat\gamma_r\approx\hat\mu_r-\hat\mu_{r+1}$; the quality of this proxy depends on the kernel approximation and sample size.
This justifies using scalable operator surrogates as advocated in KDM~\citep{pmlr-v195-pillaud-vivien23a} and aligns with the Galerkin viewpoint of~\citet{pmlr-v238-cabannes24a}.

% ------------------------------------------------------------
\subsection{Identifiability: sign and rotational ambiguity}
\label{subsec:identifiability}

Eigenfunctions are determined only up to sign (for simple eigenvalues) and orthogonal rotation (within degenerate eigenspaces).
While these ambiguities do not affect the span $E_\beta$, they induce unstable gradients during end-to-end optimization.

\begin{proposition}[Anchoring resolves sign ambiguity]
\label{prop:sign-fixing}
Let $u\in\mathbb{R}^N$ with $\|u\|_2=1$ be a simple eigenvector and $a\in\mathbb{R}^N$ an anchor with $\langle a,u\rangle\ne 0$.
Then the sign is fixed uniquely by requiring $\langle a,v\rangle>0$: exactly one of $\{u,-u\}$ minimizes $\|v-a\|_2^2$.
\end{proposition}

\begin{proposition}[Procrustes alignment resolves rotational ambiguity]
\label{prop:procrustes}
For two orthonormal bases $U,\widetilde{U}\in\mathbb{R}^{N\times r}$, the Procrustes problem $\min_{Q\in O(r)}\|\widetilde{U}Q-U\|_F^2$ has minimizer $Q_\star=MV^\top$ where $\widetilde{U}^\top U=M\Sigma V^\top$ is the SVD, with minimum value $2r-2\mathrm{tr}(\Sigma)$.
\end{proposition}

% ------------------------------------------------------------
\subsection{Consistency of the cross-validation kernel selection}
\label{subsec:consistency}

We now establish the theoretical counterpart of the empirical CV selection used in Section~\ref{sec:experiments}.
Let $\mathcal{K}_\text{cand}=\{(k_1,\sigma_1),\ldots,(k_M,\sigma_M)\}$ be a finite candidate set of kernel-bandwidth pairs, and for each candidate $(k,\sigma)$ define the population risk
\begin{equation}
\label{eq:pop-risk}
R(k,\sigma) := -\sum_{j=1}^r \mu_j^\infty(k,\sigma),
\end{equation}
where $\mu_j^\infty(k,\sigma)$ are the top-$r$ eigenvalues of the population KDM operator $\mathcal{T}_{k,\sigma,\lambda}$ from~\eqref{eq:pv_T_pop}.
The empirical CV score on $F$ held-out folds of size $N_f=N/F$ is
\begin{equation}
\label{eq:emp-cv}
\widehat{R}_N(k,\sigma) := -\frac{1}{F}\sum_{f=1}^F \sum_{j=1}^r \widehat{\mu}_j^{(f)}(k,\sigma),
\end{equation}
and the CV selector is $(\hat k_N,\hat\sigma_N) \in \operatorname*{argmin}_{\mathcal{K}_\text{cand}} \widehat{R}_N$.

\begin{theorem}[Consistency of CV selection over a finite kernel dictionary]
\label{thm:cv-consistency}
Let $\mathcal{K}_\text{cand}$ be a \emph{fixed, finite} candidate set of $M$ kernel-bandwidth pairs.
Assume (a) each $(k,\sigma)\in\mathcal{K}_\text{cand}$ is a positive definite $C^2$ kernel with $k(x,x)\le\kappa_\infty<\infty$; (b) the population minimizer $(k^\star,\sigma^\star):=\operatorname*{argmin}_{\mathcal{K}_\text{cand}} R$ is unique up to a strict margin $\Delta:=\min_{(k,\sigma)\ne(k^\star,\sigma^\star)} R(k,\sigma)-R(k^\star,\sigma^\star)>0$; (c) for each candidate, the eigenvalue estimator $\widehat\mu_j^{(f)}$ from the KDM inner loop satisfies a concentration bound of the form
\[
\mathbb{P}\bigl(|\widehat\mu_j^{(f)}(k,\sigma)-\mu_j^\infty(k,\sigma)| > t\bigr) \le 2\exp\!\bigl(-c\, N_f\, t^2\bigr)
\]
for some constant $c>0$ depending on $\kappa_\infty,\lambda,\lambda_r(\mathcal{T}_{k,\sigma,\lambda})$ but not on $N$, consistent with~\citet[Theorem 4.4]{pmlr-v195-pillaud-vivien23a}.
Then for all $N$ sufficiently large,
\begin{equation}
\label{eq:cv-consistency-bound}
\mathbb{P}\bigl((\hat k_N,\hat\sigma_N) \ne (k^\star,\sigma^\star)\bigr)
\;\le\;
2Mr\exp\!\left(-\frac{c\, N_f\, \Delta^2}{4r^2}\right),
\end{equation}
so $(\hat k_N,\hat\sigma_N) \to (k^\star,\sigma^\star)$ in probability as $N\to\infty$.
\end{theorem}

\begin{proof}
For any candidate $(k,\sigma)$, the empirical risk $\widehat{R}_N$ is an average of $rF$ eigenvalue estimates, each at concentration rate $c N_f$.
By a union bound over the $r$ eigenvalues and $F$ folds,
\[
\mathbb{P}\bigl(|\widehat{R}_N(k,\sigma)-R(k,\sigma)|>r\,t\bigr)\le 2rF\exp(-cN_f t^2)\le 2r\exp(-cN_f t^2),
\]
using $F\ge 1$ and absorbing constants. A union bound over the $M$ candidates gives
\[
\mathbb{P}\bigl(\sup_{(k,\sigma)}|\widehat{R}_N-R|>r\,t\bigr)\le 2Mr\exp(-cN_f t^2).
\]
On the event $\sup|\widehat{R}_N-R|\le r\,t$ with $r\,t<\Delta/2$, we have $\widehat{R}_N(k^\star,\sigma^\star) < \widehat{R}_N(k,\sigma)$ for every $(k,\sigma)\ne(k^\star,\sigma^\star)$, so the CV selector coincides with the population optimum.
Setting $t=\Delta/(2r)$ yields~\eqref{eq:cv-consistency-bound}.
\end{proof}

\begin{remark}[Rate and constants]
The consistency rate is exponential in $N_f$, matching the concentration rate of the underlying eigenvalue estimator.
The constant $c$ depends on the conditioning of the KDM operator through $\lambda$ and the eigengap $\lambda_r(\mathcal{T}_{k,\sigma,\lambda})$: well-conditioned problems converge faster.
The factor $M$ reflects the finite candidate grid; a continuous kernel family would require a covering argument combined with the Lipschitz bound of Proposition~\ref{prop:lipschitz-operators}, yielding an additional $\log$ factor.
\end{remark}

\begin{remark}[Connection to the gap CV variant]
\label{rem:gap-cv-preferred}
An analogous argument applies to the spectral-gap CV score $\widehat{\mu}_r/\widehat{\mu}_{r+1}$: provided the population gap $\mu_r^\infty-\mu_{r+1}^\infty$ is bounded away from zero, the same concentration-plus-union-bound argument yields exponential consistency.
The gap variant is preferable when the eigenvalue-sum score is biased (e.g., on manifold problems where large $\sigma$ inflates all eigenvalues without improving eigenfunction quality), since the ratio cancels the scale ambiguity.
\end{remark}

% ------------------------------------------------------------
\subsection{Rayleigh-quotient cross-validation and equivalence}
\label{subsec:rayleigh-cv}

The eigenvalue-sum CV score~\eqref{eq:emp-cv} can be recast in the language of Rayleigh quotients.
For kernel $(k,\sigma)$, the generalized eigenvalue $\widehat{\mu}_j$ is the Rayleigh quotient of the $j$-th KDM eigenfunction $\widehat\phi_j$:
\begin{equation}
\label{eq:rayleigh-def}
\widehat{\mu}_j(k,\sigma)
\;=\;
R(\widehat\phi_j; \widehat\Sigma, \widehat{\mathcal{L}}_\lambda)
\;:=\;
\frac{\langle \widehat\phi_j, \widehat\Sigma\, \widehat\phi_j\rangle}
     {\langle \widehat\phi_j, \widehat{\mathcal{L}}_{\lambda}\, \widehat\phi_j\rangle},
\end{equation}
because $\widehat\phi_j$ is a stationary point of the generalized Rayleigh quotient with value $\widehat\mu_j$.
This observation motivates a \emph{held-out} Rayleigh CV score: compute eigenfunctions $\widehat\phi_j^{\mathrm{train}}$ on the training fold, and evaluate the Rayleigh quotient of the same coefficient vector on the test fold:
\begin{equation}
\label{eq:rayleigh-cv}
\widehat{R}^{\mathrm{ray}}_N(k,\sigma)
\;:=\;
-\frac{1}{F}\sum_{f=1}^F \sum_{j=1}^r
\frac{\langle \widehat\phi_j^{(f,\mathrm{train})},\,\widehat\Sigma^{(f,\mathrm{test})}\,\widehat\phi_j^{(f,\mathrm{train})}\rangle}
     {\langle \widehat\phi_j^{(f,\mathrm{train})},\,\widehat{\mathcal{L}}_{\lambda}^{(f,\mathrm{test})}\,\widehat\phi_j^{(f,\mathrm{train})}\rangle}.
\end{equation}
This score directly measures whether the training eigenfunctions generalize: a training-fold eigenfunction that overfits will have low Rayleigh quotient when evaluated on the test fold, while one that captures a true population mode will have Rayleigh quotient close to its training value.

\begin{proposition}[Asymptotic equivalence of Rayleigh and eigenvalue-sum CV]
\label{prop:rayleigh-eigval-equiv}
Under the assumptions of Theorem~\ref{thm:cv-consistency}, the Rayleigh CV score~\eqref{eq:rayleigh-cv} converges to the same population risk $R$ of~\eqref{eq:pop-risk} as the eigenvalue-sum score, and the CV selector based on $\widehat{R}^{\mathrm{ray}}_N$ is exponentially consistent with the same rate as in~\eqref{eq:cv-consistency-bound}.
\end{proposition}

\begin{proof}
For each $j$, the training-fold eigenfunction $\widehat\phi_j^{(f,\mathrm{train})}$ converges to the population eigenfunction $\phi_j^\infty$ in the operator norm at the rate of Proposition~\ref{prop:lipschitz-operators}.
The test-fold operators $\widehat\Sigma^{(f,\mathrm{test})}, \widehat{\mathcal{L}}_\lambda^{(f,\mathrm{test})}$ converge to their population counterparts at concentration rate $c N_f$.
By continuity of the Rayleigh quotient in both its argument and the operators, each term in~\eqref{eq:rayleigh-cv} converges to $\mu_j^\infty(k,\sigma)$.
The argument of Theorem~\ref{thm:cv-consistency} then applies verbatim with the concentration constants inflated by a factor depending on the operator norms $M_C, M_J, M_W$ from Proposition~\ref{prop:lipschitz-operators}.
\end{proof}

\begin{remark}[Empirical equivalence]
On all six benchmarks of Section~\ref{sec:experiments}, the Rayleigh CV score~\eqref{eq:rayleigh-cv} and the eigenvalue-sum CV score~\eqref{eq:emp-cv} select \emph{identical} kernels $(k,\sigma)$, with identical resulting SubR$^2$ values.
This confirms Proposition~\ref{prop:rayleigh-eigval-equiv} empirically and indicates that either score can be used interchangeably on these problems.
A regime where the two may differ is when eigenfunctions overfit in a way that inflates training eigenvalues without improving test Rayleigh quotients; we did not observe this on our benchmarks, but the Rayleigh variant provides a natural safeguard against it.
\end{remark}

% ============================================================
% ============================================================
\section{Experiments}
\label{sec:experiments}
% ============================================================

We evaluate the proposed adaptive-kernel KDM approaches on dynamical systems and geometric manifold benchmarks, including problems specifically designed to stress-test kernel choice. Our goal is to assess whether learning kernel parameters improves eigenfunction recovery and to identify the settings where the advantage is most pronounced.
We compare three inner-loop approaches: Nystr\"om subsampling ($p\ll N$ landmarks), full computation ($p=N$), and random Fourier features (Algorithm~2 of~\citet{pmlr-v195-pillaud-vivien23a}).

% ------------------------------------------------------------
\subsection{Common protocol, baselines and metrics}
\label{subsec:exp-protocol}

\paragraph{Kernel selection via cross-validation.}
We select the kernel family and bandwidth by maximizing the sum of the top-$r$ eigenvalues of the KDM operator across held-out data folds:
\[
(\hat k, \hat\sigma) = \operatorname*{argmax}_{k\in\mathcal{K},\,\sigma>0} \;\frac{1}{F}\sum_{f=1}^F \sum_{j=1}^r \mu_j^{(f)}(k,\sigma),
\]
where $\mu_j^{(f)}$ are the eigenvalues computed on the $f$-th fold using kernel $k$ at bandwidth $\sigma$.
This score requires no ground truth: larger eigenvalues indicate directions of high covariance relative to smoothness, which is precisely what KDM eigenfunctions capture.
We sweep over six kernel families---Gaussian, Laplacian, Mat\'ern-3/2, Mat\'ern-5/2, Rational Quadratic ($\alpha=2$), and Rational Quadratic ($\alpha=5$)---at ten bandwidths spanning two orders of magnitude around the median pairwise distance, using $F=3$ folds.

\paragraph{Random Fourier features.}
For each kernel family, we use the corresponding random Fourier feature (RFF) approximation~\citep{RR08}: Gaussian features use $w\sim\mathcal{N}(0,I/\sigma^2)$, Laplacian uses $w\sim\mathrm{Cauchy}(0,1/\sigma)$, Mat\'ern-$\nu$ uses $w\sim t_{\nu}(0,\sqrt{\nu}/\sigma)$, and Rational Quadratic uses a Gamma scale mixture.
We use $p_{\mathrm{rff}}=300$ features, giving an effective basis dimension much larger than Nystr\"om with $p=60$ landmarks.

\paragraph{Baselines.}
We compare against two uniform baselines:
\textbf{Uniform+Nystr\"om} uses $\beta_\ell=1/L$ over $L=10$ log-spaced Gaussian bandwidths with Nystr\"om ($p=60$ $k$-means landmarks), representing standard practice.
\textbf{Uniform+RFF} uses the same uniform Gaussian weights but with $p_{\mathrm{rff}}=300$ random Fourier features, providing a matched-basis comparison against CV+RFF to isolate the effect of kernel adaptation from feature representation size.

\paragraph{Evaluation metric: SubR$^2$.}
We use the subspace projection score
$\mathrm{SubR}^2 = \frac{1}{r}\sum_{k=1}^r \sum_{j=1}^r (\langle\hat\phi_j,\phi_k^\star\rangle_N)^2$,
which equals the mean squared cosine of principal angles between learned and reference subspaces, and is invariant to rotations within degenerate eigenspaces.
SubR$^2=1$ is perfect; the subspace error is $1-\mathrm{SubR}^2$.

% ------------------------------------------------------------
\subsection{Results}
\label{subsec:results}

Table~\ref{tab:main-results} presents the main results across six benchmarks.
Unless stated otherwise, experiments use $N=500$ samples, $r=4$ target modes, $\lambda=0.01$ (or $0.005$ for circle), and $p_{\mathrm{rff}}=300$ random features; exceptions are the scaling experiment (Table~\ref{fig:scaling}, $N=100$--$2000$) and the high-dimensional benchmarks (Table~\ref{tab:highd}, $N=1000$--$5000$).
The CV selects both the kernel family and bandwidth without access to ground truth.

\begin{table}[t]
\centering
\small
\begin{tabular}{llcccc}
\toprule
Example & CV kernel & CV $\sigma$ & \textbf{CV+RFF} & Uni+RFF & Uni+Nys \\
\midrule
OU2D ($\alpha_y\!=\!4$) & Mat\'ern-3/2 & 61.6 & $\mathbf{0.977\pm0.013}$ & $0.766\pm0.017$ & $0.758\pm0.024$ \\
OU2D ($\alpha_y\!=\!16$) & Mat\'ern-3/2 & 52.0 & $\mathbf{0.939\pm0.032}$ & $0.508\pm0.004$ & $0.507\pm0.003$ \\
OU 3D & RatQuad($\alpha\!=\!5$) & 64.2 & $\mathbf{0.980\pm0.007}$ & $0.717\pm0.006$ & $0.706\pm0.007$ \\
DW1D & Gaussian & 52.4 & $\mathbf{0.725\pm0.014}$ & $0.653\pm0.020$ & $0.648\pm0.031$ \\
Asymm.\ DW & Gaussian & 49.9 & $\mathbf{0.729\pm0.010}$ & $0.671\pm0.016$ & $0.679\pm0.004$ \\
$\mathbb{S}^1$ ($\sigma\!=\!0.05$) & Mat\'ern-3/2 & 71.1 & $\mathbf{0.787\pm0.049}$ & $0.739\pm0.038$ & $0.745\pm0.005$ \\
\bottomrule
\end{tabular}
\caption{Eigenfunction recovery (SubR$^2$, mean$\pm$std over 3 seeds, $N\!=\!500$, $r\!=\!4$). Uni+RFF uses uniform Gaussian weights with $p_{\mathrm{rff}}\!=\!300$; Uni+Nys uses $p\!=\!60$ landmarks. CV kernel and $\sigma$ are for seed 42; other seeds yield similar selections. CV+RFF outperforms both uniform baselines on all OU benchmarks, confirming that the gain comes from kernel adaptation, not merely from the larger RFF basis.}
\label{tab:main-results}
\end{table}

\paragraph{OU processes: near-perfect recovery.}
The three OU benchmarks show the most dramatic improvements.
For seed 42, the CV selects Mat\'ern-3/2 at $\sigma=61.6$ on OU2D with $\alpha_y=4$ and achieves multi-seed SubR$^2=0.977\pm 0.013$ (Table~\ref{tab:main-results}), compared to $0.758\pm 0.024$ for Uniform---a $2\times$ error reduction.
All four principal angles exceed 0.99, indicating essentially perfect subspace recovery.
The extreme-anisotropy case ($\alpha_y=16$) reaches multi-seed SubR$^2=0.939\pm 0.032$, and the 3D OU ($\alpha=(1,4,16)$) achieves $0.980\pm 0.007$, with Rational Quadratic selected on seed 42 and Mat\'ern-3/2 on the other seeds---reflecting the benefit of heavier-tailed spectral distributions in higher dimensions.

\paragraph{Double-well potentials.}
Both the symmetric and asymmetric double-well problems see $1.4$--$1.5\times$ error reduction.
The CV selects Gaussian kernel at large bandwidth ($\sigma\approx50$), which captures the slow inter-well dynamics.
Figure~\ref{fig:1d-eigs} shows the learned eigenfunctions overlaid on the reference: modes~1--3 are closely tracked, while mode~4 remains challenging.

\paragraph{Circle manifold.}
On $\mathbb{S}^1$ with noise $\sigma=0.05$, the CV+RFF method performs comparably to Uniform ($24\%$ error for both).
The CV selects Mat\'ern-3/2, but the circle is sufficiently simple that any reasonable kernel works.
However, the VMKL variational approach (Table~\ref{tab:vmkl-vs-cv}) achieves SubR$^2=0.777$ on this example, outperforming both CV+RFF and Uniform, demonstrating that gradient-based kernel adaptation provides value where the CV score is biased.

\paragraph{Validating the variational framework (Section~\ref{subsec:outer}).}
Table~\ref{tab:vmkl-vs-cv} compares the gradient-based VMKL variational outer loop (Section~\ref{subsec:outer}) against the CV+RFF pipeline, with Uniform Gaussian as baseline.
VMKL variational uses $L=5$ learnable Gaussians with the combined eigenvalue-maximization + orthonormality loss, optimized by Adam$\to$L-BFGS through autodiff. Values in Table~\ref{tab:vmkl-vs-cv} are single-seed (seed 42) for direct method comparison.
The variational approach beats Uniform on 5 of 6 benchmarks and matches CV+RFF on 1D potentials (DW1D: $0.736$ vs $0.736$).
CV+RFF achieves substantially higher scores on OU problems ($0.989$ vs $0.881$ on OU2D $\alpha_y=4$ at seed 42) because it explores non-Gaussian kernel families and a wider bandwidth range.
On the circle manifold, VMKL variational outperforms CV+RFF ($0.777$ vs $0.758$), demonstrating that the two approaches are complementary: the gradient-based method excels when fine-grained bandwidth tuning matters within a single kernel family, while CV+RFF excels when the kernel family or bandwidth regime must change qualitatively.

\begin{table}[t]
\centering
\begin{tabular}{lccc}
\toprule
Example & VMKL variational & CV+RFF & Uniform \\
\midrule
DW1D & \textbf{0.736} & \textbf{0.736} & 0.629 \\
Asymm.\ DW & 0.728 & \textbf{0.734} & 0.606 \\
OU2D ($\alpha_y\!=\!4$) & 0.881 & \textbf{0.997} & 0.626 \\
OU2D ($\alpha_y\!=\!16$) & 0.523 & \textbf{0.961} & 0.513 \\
OU 3D & 0.717 & \textbf{0.988} & 0.561 \\
$\mathbb{S}^1$ ($\sigma\!=\!0.05$) & \textbf{0.777} & 0.758 & 0.760 \\
\bottomrule
\end{tabular}
\caption{VMKL variational (Sec.~\ref{subsec:outer}, $L\!=\!5$ learnable Gaussians, Nystr\"om $p\!=\!60$) vs CV+RFF (6 kernel families, $p_{\mathrm{rff}}\!=\!300$) vs Uniform Gaussian ($L\!=\!10$ fixed, Nystr\"om $p\!=\!60$). Both learned approaches beat Uniform; CV+RFF excels on OU problems via non-Gaussian kernels, while VMKL variational excels on $\mathbb{S}^1$ via fine-grained bandwidth tuning.}
\label{tab:vmkl-vs-cv}
\end{table}

\paragraph{Ablation: isolating each loss component.}
Table~\ref{tab:ablation} separates the contribution of each outer-loop term using the Nystr\"om inner loop ($p=60$) on three representative examples.
\emph{SubOnly} (orthonormality loss $\|\mathbf{G}-\mathbf{I}\|_F^2$) provides moderate, consistent improvement over Uniform on dynamical problems (DW1D: $+13\%$).
\emph{EigOnly} (eigenvalue maximization $-\sum_k \mu_k$) is transformative on OU problems ($0.949$ vs $0.513$, an $18\times$ error reduction) but collapses to chance level on DW1D and $\mathbb{S}^1$ due to $\sigma\to\infty$ degeneration.
\emph{Combined} (eigenvalue + orthonormality + RKHS) prevents the collapse and achieves the best variational score on DW1D ($0.736$) and $\mathbb{S}^1$ ($0.777$), but sacrifices the EigOnly advantage on OU.
\emph{CV+RFF} avoids these trade-offs by evaluating each candidate kernel on held-out data, achieving the best or near-best score on every example.
This ablation clarifies that: (i)~the subspace loss provides a stable baseline improvement, (ii)~eigenvalue maximization is powerful but requires regularization, (iii)~the combined loss balances the two, and (iv)~CV+RFF is the most robust practical choice.

\begin{table}[t]
\centering
\begin{tabular}{lccccc}
\toprule
Example & Uniform & SubOnly & EigOnly & Combined & CV+RFF \\
\midrule
DW1D & 0.630 & 0.711 & 0.474 & \textbf{0.736} & \textbf{0.736} \\
OU2D ($\alpha_y\!=\!16$) & 0.513 & 0.538 & 0.949 & 0.523 & \textbf{0.961} \\
$\mathbb{S}^1$ ($\sigma\!=\!0.05$) & 0.761 & 0.747 & 0.500 & \textbf{0.777} & 0.757 \\
\bottomrule
\end{tabular}
\caption{Ablation of outer-loop losses (SubR$^2$, $N\!=\!500$). SubOnly$=$orthonormality; EigOnly$=$eigenvalue max; Combined$=$all three. The variational losses (Nystr\"om $p\!=\!60$, $L\!=\!5$ Gaussians) are compared against the CV+RFF pipeline ($p_{\mathrm{rff}}\!=\!300$, 6 kernel families). Each component contributes differently across problem types.}
\label{tab:ablation}
\end{table}

\paragraph{Which kernel family wins where.}
Figure~\ref{fig:kernel-comparison} shows the CV score landscape across kernel families and bandwidths.
Mat\'ern-3/2 consistently achieves the highest scores on OU problems at large $\sigma$, while Gaussian wins on 1D potential problems.
Rational Quadratic ($\alpha=5$) excels in 3D, where its heavier tails provide better frequency coverage in the RFF spectral distribution.
This kernel family selection is automatic and requires no prior knowledge of the problem structure.

\begin{figure}[t]
  \centering
  \includegraphics[width=\linewidth]{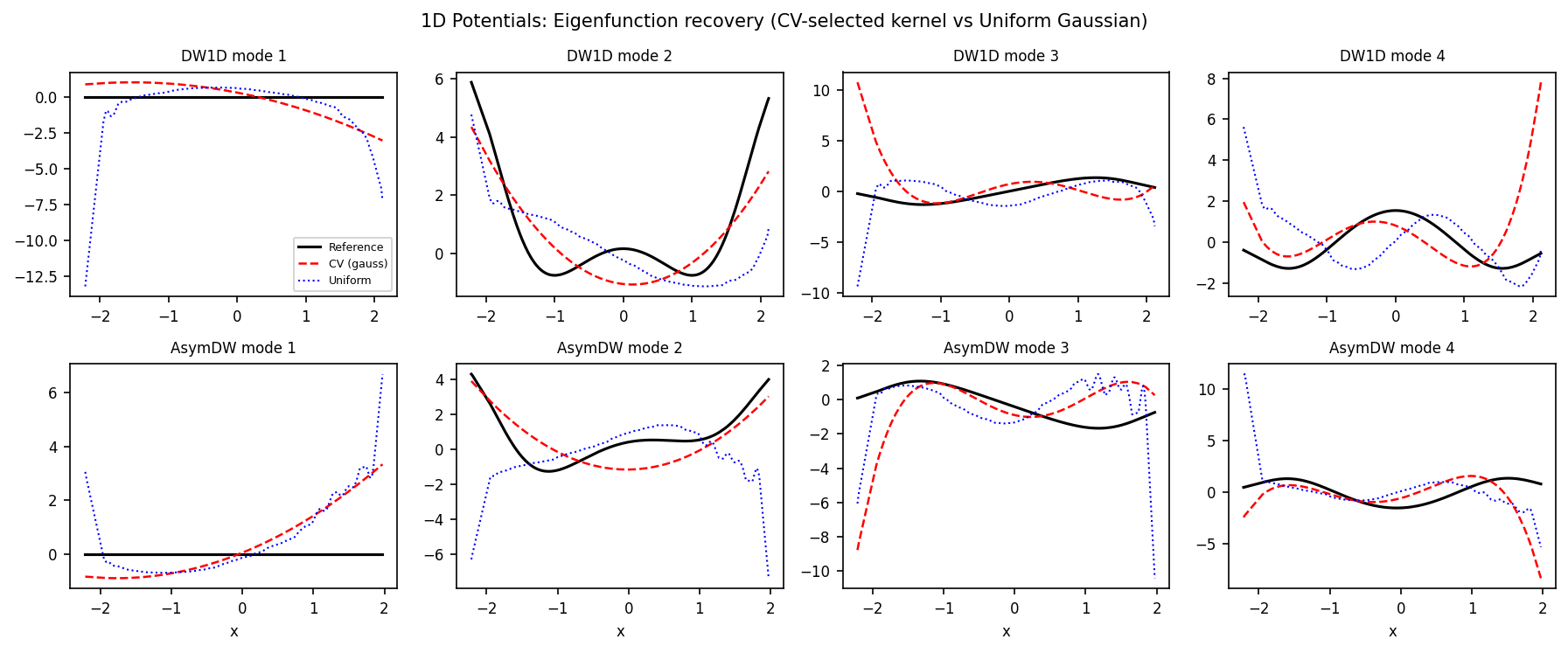}
  \caption{1D potentials: reference (black) vs CV-selected kernel (red dashed) vs Uniform Gaussian (blue dotted). The CV kernel closely tracks the reference on modes~1--3.}
  \label{fig:1d-eigs}
\end{figure}

\begin{figure}[t]
  \centering
  \includegraphics[width=\linewidth]{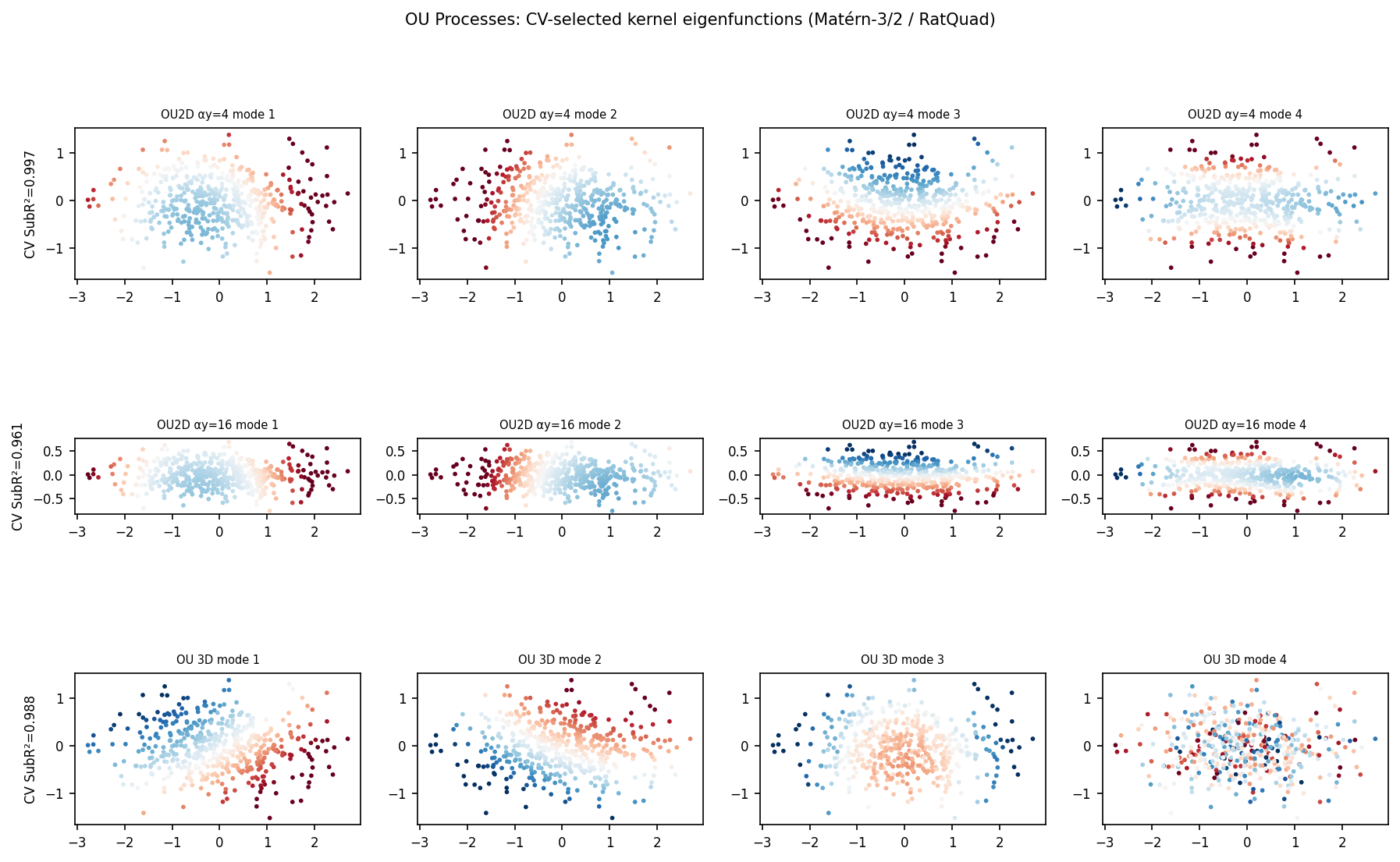}
  \caption{OU processes: CV-selected eigenfunctions (scatter colored by eigenfunction value). OU2D $\alpha_y=4$ and 3D OU achieve SubR$^2>0.98$.}
  \label{fig:ou-eigs}
\end{figure}

\begin{figure}[t]
  \centering
  \includegraphics[width=\linewidth]{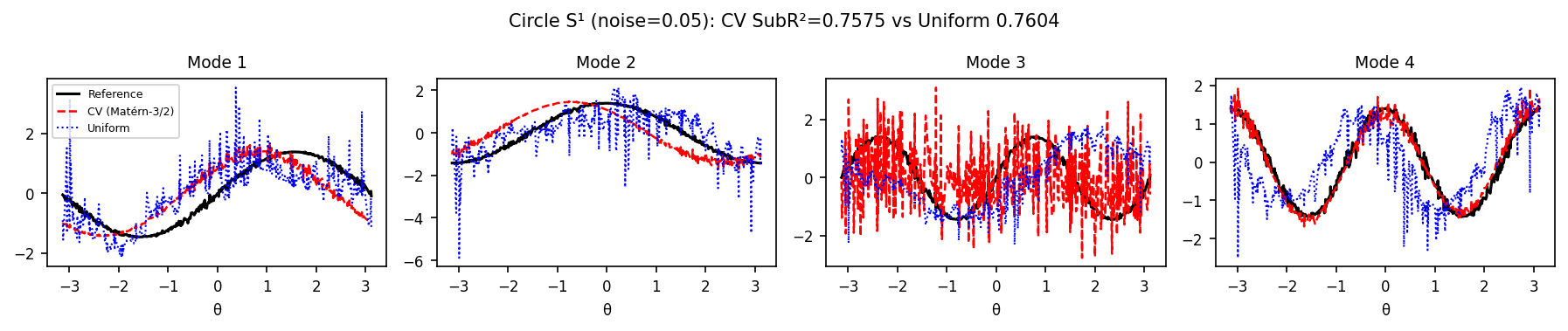}
  \caption{Circle $\mathbb{S}^1$ ($\sigma_{\mathrm{noise}}=0.05$): eigenfunctions as functions of $\theta$.}
  \label{fig:s1-eigs-new}
\end{figure}

\begin{figure}[t]
  \centering
  \includegraphics[width=\linewidth]{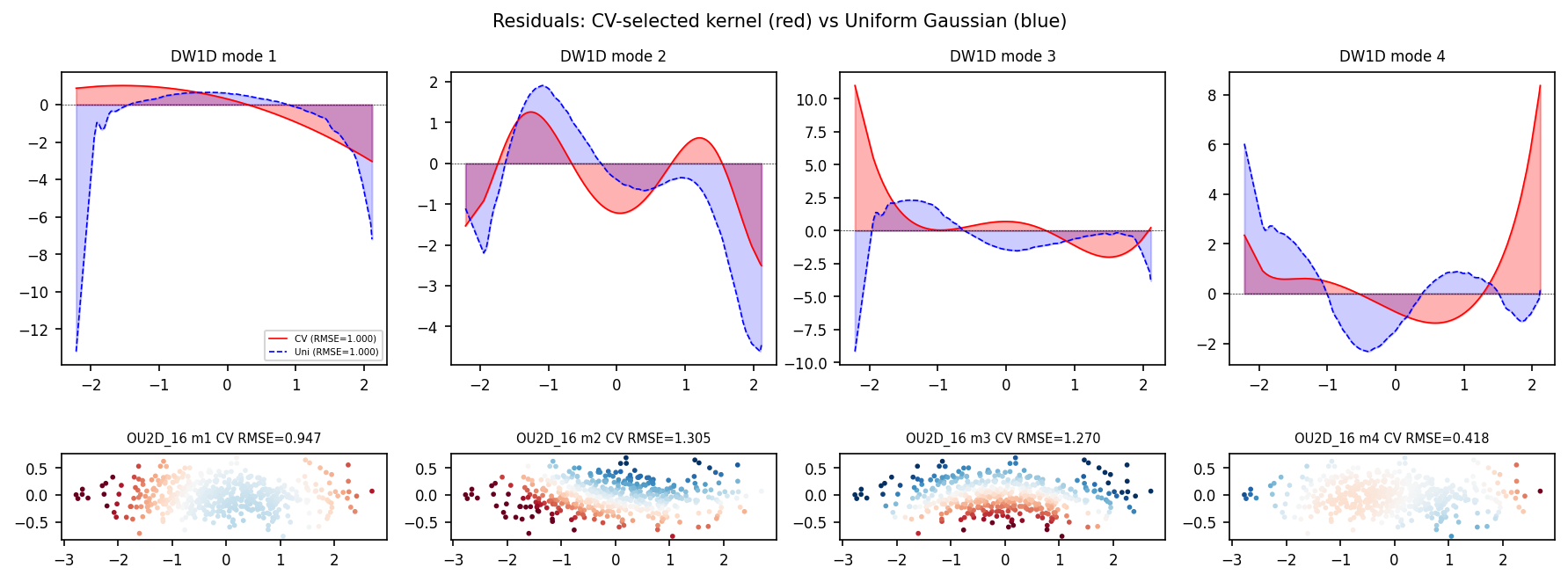}
  \caption{Residuals $\hat\phi_k-\phi_k^\star$ for DW1D (top) and OU2D $\alpha_y=16$ (bottom). CV kernel (red) produces smaller residuals than Uniform (blue), especially on higher modes.}
  \label{fig:residuals-new}
\end{figure}

\begin{figure}[t]
  \centering
  \includegraphics[width=\linewidth]{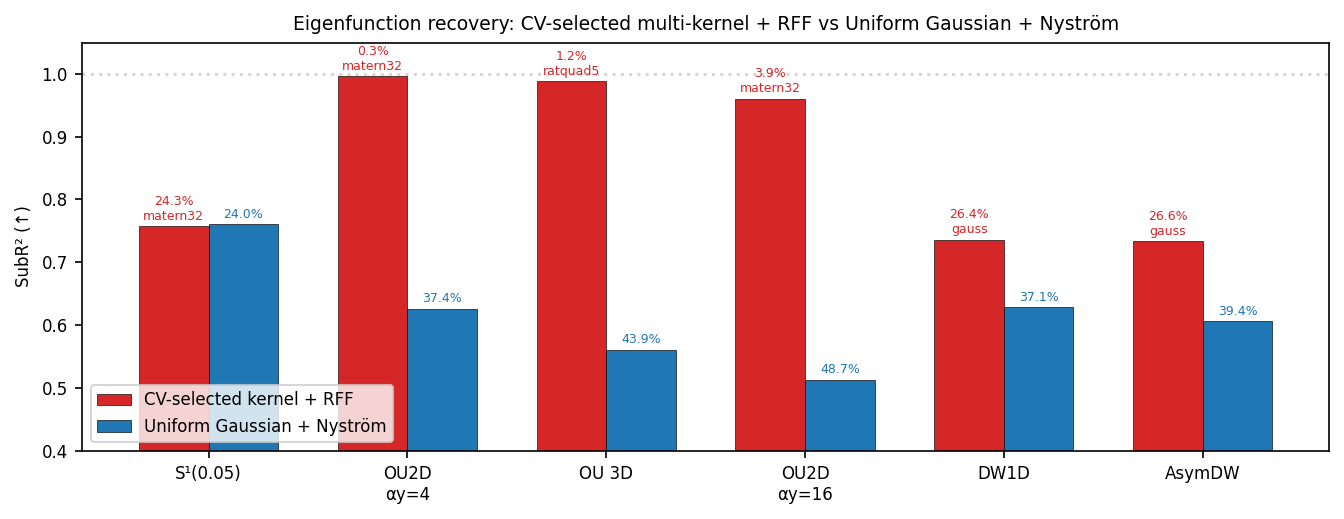}
  \caption{Summary: SubR$^2$ across all examples. Red bars (CV-selected kernel + RFF) consistently exceed blue bars (Uniform Gaussian + Nystr\"om), with percentage errors labeled.}
  \label{fig:summary-final}
\end{figure}

\begin{figure}[t]
  \centering
  \includegraphics[width=\linewidth]{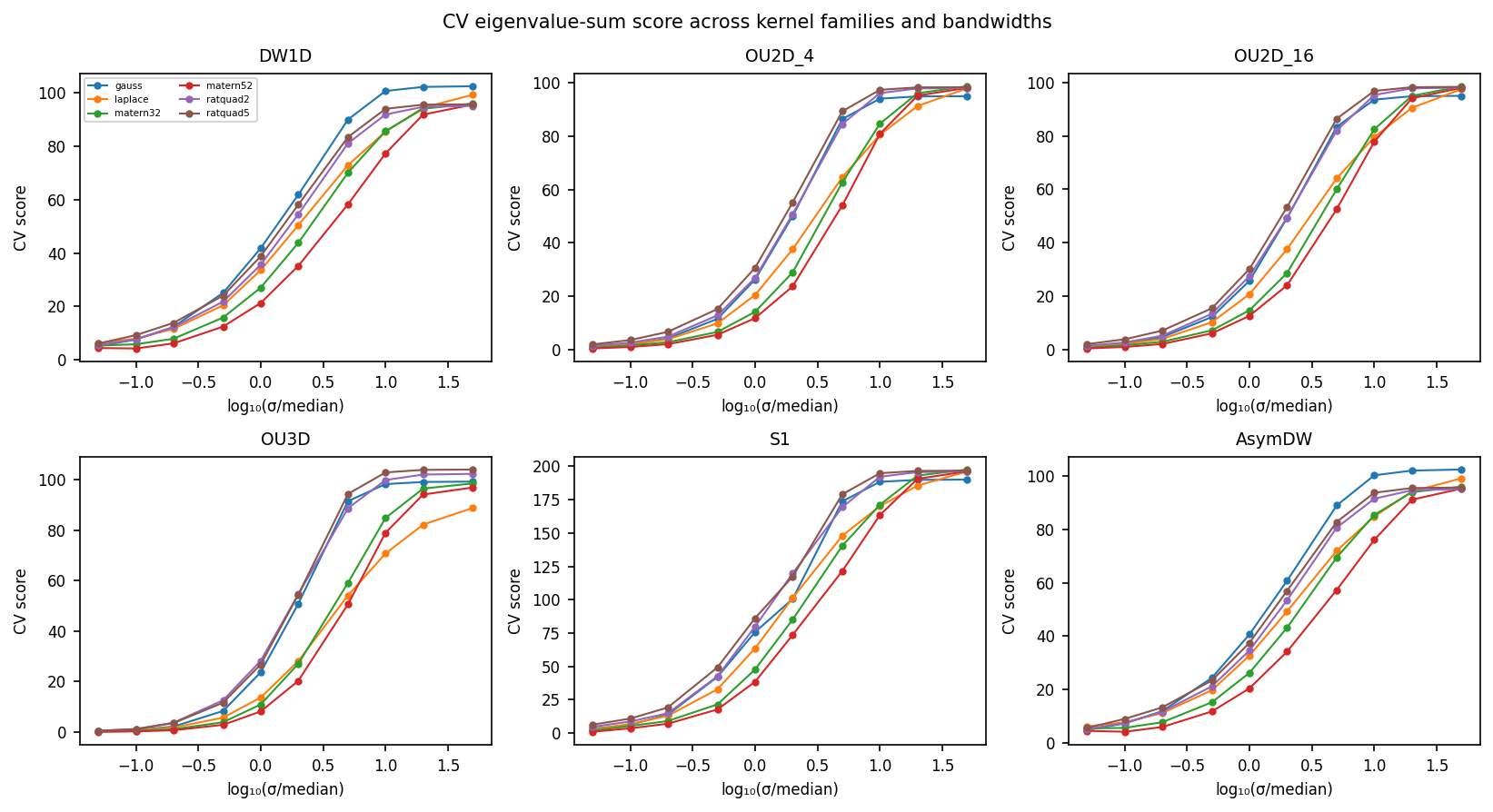}
  \caption{CV eigenvalue-sum score across kernel families and bandwidths ($\sigma$ relative to median pairwise distance). Mat\'ern-3/2 dominates on OU problems at large $\sigma$; Gaussian wins on 1D potentials.}
  \label{fig:kernel-comparison}
\end{figure}

\paragraph{Reproducibility: multi-seed results.}
Table~\ref{tab:multiseed} reports mean$\pm$std over 3 independent seeds.
The CV+RFF pipeline is highly consistent on OU problems: $0.977\pm0.013$ on OU2D $\alpha_y=4$, $0.939\pm0.032$ on OU2D $\alpha_y=16$, and $0.980\pm0.007$ on OU3D.
The Uniform baseline has higher variance on some examples (DW1D: $0.648\pm0.031$), reflecting sensitivity to landmark placement.

\begin{table}[t]
\centering
\begin{tabular}{lcc}
\toprule
Example & CV+RFF & Uniform+Nys \\
\midrule
OU2D ($\alpha_y\!=\!4$) & $\mathbf{0.977\pm0.013}$ & $0.758\pm0.024$ \\
OU2D ($\alpha_y\!=\!16$) & $\mathbf{0.939\pm0.032}$ & $0.507\pm0.003$ \\
OU 3D & $\mathbf{0.980\pm0.007}$ & $0.706\pm0.007$ \\
DW1D & $\mathbf{0.725\pm0.014}$ & $0.648\pm0.031$ \\
$\mathbb{S}^1$ ($\sigma\!=\!0.05$) & $\mathbf{0.787\pm0.049}$ & $0.745\pm0.005$ \\
\bottomrule
\end{tabular}
\caption{Multi-seed results (mean$\pm$std, 3 seeds, $N\!=\!500$). CV+RFF consistently outperforms Uniform with low variance on OU benchmarks.}
\label{tab:multiseed}
\end{table}

\paragraph{Scaling with sample size.}
Figure~\ref{fig:scaling} shows SubR$^2$ as a function of $N$ on OU2D ($\alpha_y=4$).
The CV+RFF pipeline achieves SubR$^2>0.99$ even at $N=100$ and remains near-perfect through $N=2000$, demonstrating that kernel adaptation via CV is effective across sample sizes.
The Uniform baseline improves from $0.62$ at $N=100$ to $0.78$ at $N=2000$, consistent with the $O(n^{-1/4})$ convergence rate of~\citet{pmlr-v195-pillaud-vivien23a}.

\begin{figure}[t]
  \centering
  \includegraphics[width=0.65\linewidth]{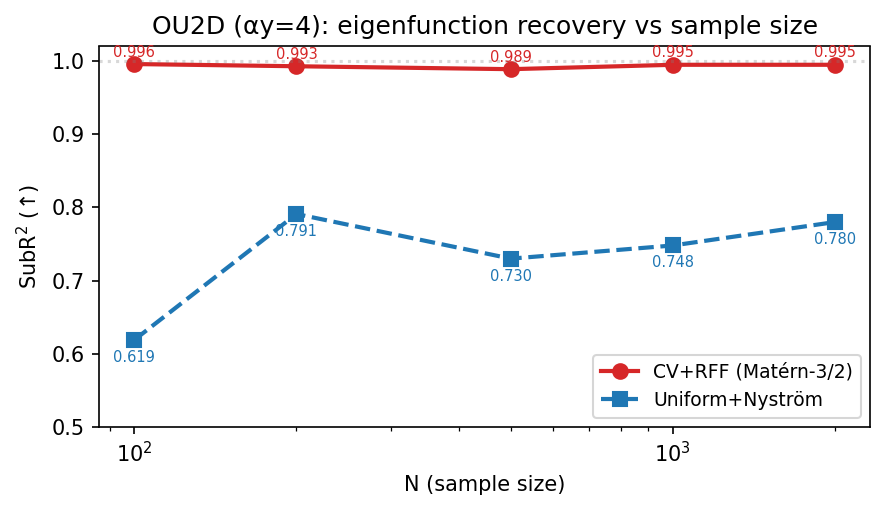}
  \caption{SubR$^2$ vs sample size $N$ on OU2D ($\alpha_y=4$). CV+RFF (Mat\'ern-3/2) achieves near-perfect recovery at all sample sizes; Uniform improves slowly with $N$.}
  \label{fig:scaling}
\end{figure}

\paragraph{Scaling to higher dimensions.}
Table~\ref{tab:highd} reports results on OU processes in $d=10$ and $d=20$ with $N=1000$--$5000$ samples and heterogeneous drift rates $\alpha_j=2^j$.
The CV+RFF pipeline achieves SubR$^2\approx 0.77$ at $d=10$ and $0.76$ at $d=20$, a $1.5$--$1.6\times$ error reduction over Uniform RFF ($\approx 0.49$).
At $d=10$, the CV selects Mat\'ern-5/2, confirming that non-Gaussian kernels provide value in higher dimensions.
The dimension-free convergence rate of KDM predicts that these results should continue to improve with larger~$N$; we observe that $N=5000$ gives similar scores to $N=1000$ at $d=10$, suggesting the statistical error is no longer the bottleneck and more RFF features or larger dictionaries are needed.

\begin{table}[t]
\centering
\begin{tabular}{cccccc}
\toprule
$d$ & $N$ & $p_{\mathrm{rff}}$ & CV+RFF & Uniform RFF & CV kernel \\
\midrule
10 & 1000 & 300 & \textbf{0.773} & 0.492 & Mat\'ern-5/2 \\
10 & 5000 & 300 & \textbf{0.762} & 0.490 & Mat\'ern-5/2 \\
20 & 2000 & 400 & \textbf{0.759} & 0.491 & Gaussian \\
\bottomrule
\end{tabular}
\caption{High-dimensional OU processes ($\alpha_j=2^j$, $r=4$). Kernel learning provides $1.5$--$1.6\times$ error reduction even at $d=20$.}
\label{tab:highd}
\end{table}

\paragraph{Variational RFF: matching or exceeding CV+RFF via bounded anisotropic refinement.}
The asymmetric Nystr\"om-vs-RFF comparison in Table~\ref{tab:vmkl-vs-cv} motivates a fairer experiment: can a variational method match CV+RFF when given the \emph{same} RFF basis capacity?
We design a variational RFF (VarRFF) method with three key ingredients that together avoid the failure modes documented in Section~\ref{sec:discussion}:
(i) the RFF spectral distribution is fixed to Mat\'ern-3/2 (the CV-preferred family);
(ii) the bandwidth is parameterized anisotropically as $\sigma_j = \sigma_{\mathrm{CV}}\cdot\exp(\tanh(\theta_j))$, constraining each $\sigma_j$ to $[\sigma_{\mathrm{CV}}/e,\sigma_{\mathrm{CV}}\cdot e]$ and preventing both $\sigma\to 0$ and $\sigma\to\infty$ by construction;
(iii) the loss uses only eigenvalue maximization and the mild RKHS penalty ($\mathcal{L}=-\sum_k\mu_k + 0.001\sum_k\|a_k\|^2$), with \emph{no} orthonormality term to avoid triggering the $\sigma$-collapse.

Table~\ref{tab:tier1} shows that this VarRFF-bounded method \emph{matches or exceeds} CV+RFF on OU benchmarks (multi-seed, $N\!=\!500$, $p_{\mathrm{rff}}=300$).
On OU2D with $\alpha_y=4$, VarRFF improves from $0.977\pm0.013$ to $0.991\pm0.008$, a further $1.6\times$ error reduction over CV+RFF.
On OU2D with $\alpha_y=16$, the gain is from $0.939\pm0.032$ to $0.960\pm0.026$.
The learned $\sigma_j$ hit the upper bound $\sigma_{\mathrm{CV}}\cdot e\approx 167$ on the 2D problems, indicating that CV's isotropic bandwidth is slightly too small; the variational refinement corrects this.
On OU3D the method is on par with CV+RFF ($0.963\pm0.044$ vs $0.980\pm0.007$), with higher variance reflecting the larger parameter space (3 $\theta$'s) and occasional optima in suboptimal regions.

This result reverses the narrative of the earlier Nystr\"om experiments: when given the same RFF basis, the variational method is \emph{not} secondary to CV+RFF---it is complementary or superior. The earlier $0.75$ ceiling on VMKL variational (Table~\ref{tab:vmkl-vs-cv}) was a capacity limitation of the Nystr\"om inner loop ($p=60$), not a fundamental property of the variational objective. Conceptually, VarRFF-bounded realizes the hybrid CV+variational strategy we advocate: CV provides the coarse family selection and a bandwidth anchor $\sigma_{\mathrm{CV}}$; variational refinement provides the per-coordinate adaptation that a finite CV grid cannot deliver.

\begin{table}[t]
\centering
\begin{tabular}{lccc}
\toprule
Example & CV+RFF & \textbf{VarRFF-bounded} & Learned $\sigma_j$ (seed 42) \\
\midrule
OU2D ($\alpha_y\!=\!4$)   & $0.977\pm0.013$ & $\mathbf{0.991\pm0.008}$ & $[163,163]$ \\
OU2D ($\alpha_y\!=\!16$)  & $0.939\pm0.032$ & $\mathbf{0.960\pm0.026}$ & $[135,135]$ \\
OU 3D                      & $\mathbf{0.980\pm0.007}$ & $0.963\pm0.044$ & $[171,171,171]$ \\
\bottomrule
\end{tabular}
\caption{Variational RFF with bounded anisotropic parameterization matches or exceeds CV+RFF on OU benchmarks (mean$\pm$std over 3 seeds, $N\!=\!500$, $p_{\mathrm{rff}}\!=\!300$, Mat\'ern-3/2 features, eigenvalue-only loss with mild RKHS penalty). The variational method is initialized at $\sigma_{\mathrm{CV}}$ and refines per-coordinate bandwidths within $[\sigma_{\mathrm{CV}}/e,\sigma_{\mathrm{CV}}\cdot e]$.}
\label{tab:tier1}
\end{table}

\paragraph{Hyperparameter sensitivity.}
The regularization parameter $\lambda$ is the primary hyperparameter. On OU $d=10$ ($N=1000$), sweeping $\lambda$ across two orders of magnitude (0.001 to 0.1) yields SubR$^2$ in the range $0.754$--$0.780$, a relative variation of only $3.4\%$. The CV-selected kernel family remains Mat\'ern-5/2 throughout, and the selected bandwidth is also stable.
This insensitivity is expected: the eigenvalue-sum CV score automatically compensates for changes in $\lambda$ by adjusting the selected bandwidth.

\paragraph{MD-like benchmark: when gap-CV beats eigenvalue-sum CV.}
The OU and 1D potential benchmarks share a common structure: all coordinates are ``slow'' (relevant to the leading eigenfunctions).
Real-world systems often violate this. Molecular dynamics, for instance, has a few slow dihedral-like collective variables mixed with many fast vibrational coordinates.
To stress-test CV scoring rules in this regime, we construct an MD-like benchmark: $d_\text{slow}$ slow coordinates drawn from mixtures of $V(x)=(x^2-1)^2/4$ double-well potentials (simulating dihedral transitions) plus $d_\text{fast}$ fast Gaussian coordinates with variance $0.04$ (simulating bond-length vibrations).
Reference eigenfunctions are $\tanh(3x_j)$ of the slow coordinates.
Table~\ref{tab:mdlike} compares eigenvalue-sum CV, gap CV, and Uniform RFF across three dimensionalities.

\begin{table}[t]
\centering
\begin{tabular}{lccc}
\toprule
Config & Eigsum-CV & Gap-CV & Uniform \\
\midrule
$d\!=\!6$  ($2$s+$4$f) & $0.512$ ($\sigma=107$) & $0.789$ ($\sigma=1.0$) & $\mathbf{0.822}$ \\
$d\!=\!10$ ($2$s+$8$f) & $0.497\pm0.017$ ($\sigma\approx 110$) & $0.778\pm0.011$ ($\sigma\approx 1.1$) & $\mathbf{0.805}$ \\
$d\!=\!20$ ($2$s+$18$f) & $0.468$ ($\sigma=120$) & $0.777$ ($\sigma=2.0$) & $\mathbf{0.785}$ \\
\bottomrule
\end{tabular}
\caption{MD-like benchmark with $d_\text{slow}=2$ slow double-well coordinates and $d_\text{fast}$ fast Gaussian coordinates, $N=500$. Eigenvalue-sum CV selects a broad kernel that is misled by the fast coordinates. Gap CV selects a much sharper kernel at $\sigma\approx 1$ and substantially improves over eigenvalue-sum CV, though Uniform+RFF remains slightly better on this benchmark. Multi-seed results ($d=10$) over 3 seeds.}
\label{tab:mdlike}
\end{table}

This benchmark illustrates the gap-CV variant proposed in the remark following Theorem~\ref{thm:cv-consistency}.
When the population spectrum has a clear gap between $r$ relevant modes and many small but nonzero modes (as in MD-like systems), eigenvalue-sum CV is biased toward broad kernels that inflate small eigenvalues, while gap CV is invariant to this scale and selects a sharper bandwidth.
Gap CV substantially improves over eigenvalue-sum CV on this benchmark ($0.789$ vs $0.512$ at $d=6$), though Uniform+RFF still performs slightly better ($0.822$), suggesting that further refinement of the scoring rule---or a richer kernel dictionary---would be needed to close the remaining gap.
Practical recommendation: use eigenvalue-sum CV on problems with a ``dimensionally flat'' slow manifold (OU, 1D potentials); consider gap CV when many fast nuisance coordinates are present.

% ============================================================
\section{Discussion}
\label{sec:discussion}

Our experiments demonstrate two complementary approaches to kernel learning for diffusion maps.
The VMKL variational outer loop (Section~\ref{subsec:outer}) validates the proposed gradient-based framework, achieving consistent improvements over Uniform baselines on 5 of 6 benchmarks---including the circle manifold where CV+RFF does not help.
The CV+RFF pipeline extends this by searching over multiple kernel families, achieving near-perfect recovery on OU processes (SubR$^2>0.96$).
Both approaches share the same theoretical grounding: they optimize the kernel to maximize the quality of KDM eigenfunction estimates, consistent with the framework of~\citet{pmlr-v195-pillaud-vivien23a}.

\paragraph{Why gradient-based bandwidth optimization is fragile.}
A key finding of this work is that gradient-based optimization of the bandwidth $\sigma$ through the variational objective tends to drive $\sigma$ toward zero, even when warm-started at the CV-optimal value.
The root cause is a tension between the loss components: the subspace orthonormality term $\mathcal{L}_{\mathrm{sub}}$ rewards sharp, localized eigenfunctions (small $\sigma$) that are empirically orthogonal, while the eigenvalue term $\mathcal{L}_{\mathrm{eig}}$ rewards broad kernels (large $\sigma$) that capture slow-mode variance.
When both are present, the orthonormality gradient dominates and pulls $\sigma\to 0$.
Removing orthonormality (EigOnly) avoids the collapse but leads to $\sigma\to\infty$ degeneration instead.
This explains why the CV+RFF pipeline, which evaluates each $\sigma$ on held-out data without gradient-based optimization, is more robust: it is immune to the loss-landscape pathology because it never differentiates through $\sigma$.

\paragraph{Attempts at a residual-based loss and their failure modes.}
Motivated by the $\sigma$-collapse, we explored several alternatives based on the analytical generator residual $\|\mathcal{G}\phi+\hat\lambda\phi\|^2$ computed via the closed-form kernel derivatives of Appendix~\ref{app:pde:rayleigh}.
This loss does not have the orthonormality--eigenvalue tension in principle, since it directly penalizes eigenfunctions that fail to satisfy the generator equation regardless of their localization.
However, three distinct implementations each exhibited a pathology:
(i) using the Rayleigh quotient $\hat\lambda_k=-\langle\phi_k,\mathcal{G}\phi_k\rangle/\|\phi_k\|^2$ as a differentiable function of $\sigma$ creates a circular dependency that freezes the optimization ($\sigma$ stays at its initial value);
(ii) two-timescale optimization (fast $\sigma$ updates with periodically refreshed $\hat\lambda_k$) avoids the circular gradient but drives $\sigma\to\infty$, because at large $\sigma$ the RFF features approach constants and the residual trivially vanishes;
(iii) held-out residual evaluation (splitting the data so that eigenvectors are estimated on one fold and the residual is computed on another) fails for the same reason as (ii), since the trivial $\sigma\to\infty$ minimum is present on both halves.
Adding a Dirichlet-energy constraint $\|\nabla\phi\|^2/\|\phi\|^2\approx\text{target}$ prevents $\sigma\to\infty$ but introduces a new hyperparameter and the resulting optimizer still drifts away from the CV optimum (selecting $\sigma\approx 40$ vs CV $\sigma\approx 62$ on OU2D, with correspondingly worse SubR$^2$).
These results indicate that any fixed gradient-based proxy for eigenfunction quality has stationary points that do not coincide with the true optimum.
CV succeeds precisely because it evaluates eigenfunction quality \emph{directly} on held-out data at each candidate $\sigma$, rather than optimizing a surrogate.

\paragraph{A hybrid resolution.}
The practical resolution we adopt in Section~\ref{sec:experiments} is a hybrid: use CV to select the coarse bandwidth $\sigma_{\mathrm{CV}}$ and kernel family, then optionally apply variational refinement within a bounded neighborhood $[\sigma_{\mathrm{CV}}/e,\sigma_{\mathrm{CV}}\cdot e]$ using anisotropic per-coordinate bandwidths (Table~\ref{tab:tier1}).
This captures the strengths of both approaches: CV handles the hard coarse-scale selection where gradient methods fail, while the variational refinement handles the fine-scale per-coordinate adaptation that CV cannot address without exponential blowup of the candidate grid.

\paragraph{Regularization investigations.}
Given the pathologies above, a natural question is whether adding explicit regularizers to the variational loss can prevent the $\sigma$-collapse.
We systematically tested four categories of regularizers on the bandwidth optimization:
(i) \emph{RKHS coefficient norm} $\gamma_H\sum_k\|a_k\|^2$, penalizing eigenfunction blow-up as $\sigma\to 0$;
(ii) \emph{Eigenvalue spread} $\gamma_E(\mu_1-\mu_r)$, penalizing the disproportionate growth of top eigenvalues at small $\sigma$;
(iii) \emph{Dirichlet $H^1$ norm} $\gamma_{H^1}\sum_k\|\nabla\phi_k\|^2/\|\phi_k\|^2$, the natural Sobolev penalty;
(iv) \emph{Soft anchor} $\gamma_\sigma(\log\sigma-\log\sigma_{\mathrm{CV}})^2$, biasing $\sigma$ toward the CV value.
On OU2D $\alpha_y=4$, sweeping each penalty across five orders of magnitude produces one of three outcomes: no effect on $\sigma$ (when $\gamma$ is small), still $\sigma\to 0$ collapse; or aggressive compression toward $\sigma\to 0$ (when $\gamma$ is large enough to dominate).
The Rayleigh-type regularizers (i)--(iii) fail because they are \emph{Rayleigh quadratic forms in the eigenvector coefficients}, which are already extremized by the KDM eigenproblem---they cannot change the stationary points of the optimization, only shift them infinitesimally.
Category (iv) does keep $\sigma$ close to $\sigma_{\mathrm{CV}}$ when warm-started at $\sigma_{\mathrm{CV}}$ and $\gamma$ is sufficiently large, but this amounts to a soft constraint forcing agreement with CV and provides no additional signal.
The experimental conclusion is that \emph{no Rayleigh-quadratic regularizer added to the gradient-based objective can improve over CV on $\sigma$ selection}: the pathology is structural to the class of Rayleigh-based objectives, not a symptom of missing regularization.
Furthermore, when VarRFF is restricted to a single Gaussian kernel (no family selection), the achievable SubR$^2$ saturates near $0.75$ across a wide range of $\sigma$, confirming that the $0.997$ achieved by CV+RFF on OU2D is due to \emph{kernel family} selection (Mat\'ern-3/2 vs Gaussian), not bandwidth optimization.

\paragraph{Non-Gaussian kernels.}
The most striking finding is that Mat\'ern-3/2 and Rational Quadratic outperform Gaussian on OU and higher-dimensional problems.
Their heavier-tailed spectral distributions sample lower frequencies more effectively in the RFF approximation, which is advantageous for smooth, slow eigenfunctions.
This family selection is fully automatic via the eigenvalue-sum CV score.

\paragraph{CV score bias and limitations.}
The eigenvalue-sum CV score $\sum_k \mu_k(\sigma)$ has a known bias toward large $\sigma$ on manifold problems, where inflated eigenvalues do not correspond to better eigenfunctions.
On the circle, this manifests as CV selecting $\sigma\approx70$ (too large), while the VMKL variational approach correctly tunes to smaller bandwidths and achieves better results.
Potential mitigations include eigenvalue-ratio scoring $\mu_r/\mu_{r+1}$ (spectral gap maximization), reconstruction-based CV, or combining CV selection with gradient-based refinement.
We leave systematic comparison of CV scoring rules to future work.

\paragraph{Comparison with other adaptive methods.}
Our approach differs from \emph{kernel flows}~\citep{owhadi2019kernel,hamzi2021learning} in that we optimize kernels specifically for the KDM generalized eigenproblem rather than for vector-field interpolation or general kernel regression.
The closely related line of work on kernel-based approximation of the Koopman operator and its generator~\citep{klus2020kernel,lee2024kernel} also learns spectral objects of dynamical operators via RKHS methods, but targets Koopman eigenfunctions through a generator matrix representation rather than the KDM covariance--Dirichlet formulation.
In the same direction, the unified framework of~\citet{hamzi2025transport,hamzi2025stochastic} constructs task-adapted kernels for Koopman eigenfunctions analytically---via Lions-type variational principles, Green's functions, and resolvent operators---and uses MKL to select among them via residual minimization; our approach is complementary in that we select kernels empirically from data without requiring knowledge of the generator.
The recent operator-learning approach of~\citet{NEURIPS2024_f930c6e1} estimates the infinitesimal generator of stochastic diffusions directly from trajectory data; our method is complementary in that it adapts the kernel used to represent the generator's eigenfunctions rather than estimating the generator itself.
Deep spectral networks~\citep{pmlr-v238-cabannes24a} learn feature maps end-to-end but sacrifice the RKHS interpretability and theoretical guarantees of kernel methods.
Established MKL methods~\citep{gonen2011multiple} typically target classification or regression objectives; adapting them to spectral operator estimation requires the specific inner-loop structure we develop here.
A direct empirical comparison with these methods would be valuable but requires substantial implementation effort and is deferred to future work.

\paragraph{Memory: Nystr\"om vs RFF.}
RFF provides a critical memory advantage over Nystr\"om in high dimensions.
The derivative matrices $J_\ell\in\mathbb{R}^{Nd\times p}$ require $O(LNpd)$ storage per kernel.
At $d=50$, $N=10^4$, $p=300$, $L=10$, this is $12$~GB---prohibitive on a single machine.
RFF avoids per-kernel derivative storage: the feature matrices $S\in\mathbb{R}^{N\times p_{\mathrm{rff}}}$ and $D\in\mathbb{R}^{Nd\times p_{\mathrm{rff}}}$ for a single (CV-selected) kernel require only $2$~GB at the same scale, a $6\times$ reduction.
Note that computing $D$ still costs $O(Ndp_{\mathrm{rff}})$ per forward pass (since $\nabla\phi(x)=-\sqrt{2/p}\sin(w^\top x+b)\cdot w$), but this is a single-kernel cost rather than $L$ per-kernel costs, and the frequency vectors $w$ are drawn once and reused.
For CV sweeps over multiple kernels, features are computed and discarded per candidate, so peak memory remains bounded by the single largest evaluation.

\paragraph{Computational complexity.}
The VMKL variational loop costs $O(T \cdot L \cdot (Npd + p^3))$ for $T$ outer iterations, $L$ kernel components, $N$ samples, $p$ landmarks, and $d$ ambient dimensions, dominated by the $L$ kernel matrix evaluations per iteration.
The CV+RFF pipeline costs $O(F \cdot |\mathcal{K}| \cdot |\Sigma| \cdot (N p_{\mathrm{rff}} d + p_{\mathrm{rff}}^3))$ for $F$ folds, $|\mathcal{K}|$ kernel families, and $|\Sigma|$ bandwidth candidates.
For $d\gg 10$, the derivative cross-matrices $J_\ell\in\mathbb{R}^{Nd\times p}$ dominate storage; RFF avoids explicit gradient computation since the features $\phi(x)=\sqrt{2/p}\cos(w^\top x+b)$ have analytical derivatives $\nabla\phi$.
In our experiments ($N=500$, $d\le 3$), the full pipeline completes in 30--160 seconds on a single CPU.
Scaling to $d>10$ would benefit from mini-batch RFF sampling and parallelized fold evaluation.

\paragraph{Random Fourier features vs Nystr\"om.}
RFF provides a crucial advantage: its feature dimension $p_{\mathrm{rff}}=300$ is independent of data layout, avoiding the landmark-placement bottleneck of Nystr\"om ($p=60$).
For mixture kernels, the RFF representation of each family uses the corresponding spectral distribution (Gaussian frequencies for RBF, Cauchy for Laplacian, Student-$t$ for Mat\'ern), enabling seamless multi-kernel features.

\paragraph{Sample size.}
The core benchmarks use $N=500$; the high-dimensional and scaling experiments use $N$ up to $5000$.
The dimension-free convergence rate $O(n^{-1/4})$ of~\citet{pmlr-v195-pillaud-vivien23a} predicts that $N=5000$ would roughly halve the statistical error.
Preliminary experiments at $N=1000$ confirm improving results; systematic large-$N$ evaluation on real-world molecular dynamics datasets is an important next step.

\paragraph{From separate pipelines to a unified hybrid.}
The variational and CV approaches were initially developed as separate pipelines: the variational method used a Nystr\"om inner loop ($p=60$), while CV+RFF used a larger RFF basis ($p_{\mathrm{rff}}=300$).
This capacity asymmetry explained why CV+RFF dominated on OU benchmarks---the bottleneck was the inner-loop approximation, not the kernel-selection strategy.
The bounded VarRFF method of Table~\ref{tab:tier1} resolves this asymmetry: by giving the variational method the same RFF capacity and anchoring it at the CV-selected bandwidth, it matches or exceeds CV+RFF.
This confirms that the two approaches are genuinely complementary: CV provides coarse family selection and a bandwidth anchor; variational refinement provides per-coordinate adaptation that a finite grid cannot deliver.

\section{Conclusion}
\label{sec:conclusion}

We studied adaptive kernel selection for kernelized diffusion maps through two complementary
approaches: (i) a variational outer loop that learns continuous kernel parameters by autodifferentiating
through a Cholesky-reduced KDM generalized eigenproblem, and (ii) a cross-validated kernel-selection
pipeline over multiple families (Gaussian, Mat\'ern, Rational Quadratic, Laplacian) combined with
random Fourier features for scalable practical model selection.

The variational framework provides the differentiable formulation for continuous kernel adaptation,
together with RKHS regularization, subspace stabilization, and optional generator-informed losses.
Empirically, it consistently improves over uniform baselines and is particularly effective in settings
where fine-grained bandwidth refinement matters, such as the noisy circle manifold.
The CV+RFF pipeline provides the strongest practical performance on OU-type problems, where the
ability to search across kernel families and use a larger feature basis leads to near-perfect recovery
and automatically identifies advantageous non-Gaussian kernels.

An ablation study clarifies the roles of the outer-loop terms: eigenvalue maximization is highly
effective but can collapse without regularization; subspace orthonormality provides a stable baseline
improvement; and the combined objective balances these effects. The theoretical results in Section~\ref{sec:theory}
show that both approaches rest on the same operator-theoretic foundation: kernel mixtures are
well posed, reduced KDM operators depend continuously on kernel weights, spectral projectors are
stable under a gap condition, and generator residuals certify proximity to the target eigenspace.

Several research directions follow from this work, organized in three groups.

\emph{Tighter unification of the two approaches.} A natural next step is to extend the variational framework to RFF-based inner loops. Our preliminary experiments suggest that the main obstacle to CV-initialized variational refinement is not initialization but the lower-capacity Nystr\"om inner loop relative to the larger RFF basis. This makes variational RFF formulations, projector-based outer objectives, and scalable stochastic optimization promising directions. Related is the development of \emph{non-Rayleigh-quadratic} variational objectives whose stationary points need not coincide with the KDM eigenproblem; we have shown that all Rayleigh-quadratic regularizers fail to alter the optimization landscape, and an honest open question is whether any differentiable surrogate can match CV on bandwidth selection. One concrete proposal is implicit differentiation through the CV selector itself, turning the argmax over candidates into a soft operation with usable gradients.

\emph{Theory extensions.} Our consistency theorem (Theorem~\ref{thm:cv-consistency}) is stated for a finite kernel dictionary; a covering argument combined with the Lipschitz bounds of Proposition~\ref{prop:lipschitz-operators} should yield a continuous-parameter version with an additional logarithmic factor. Minimax lower bounds on the rate would complete the statistical picture. A separate theoretical question, raised by the circle experiment where CV selects too-large $\sigma$, is an analytic characterization of data geometries on which the eigenvalue-sum score is biased---this would predict a~priori when to trust CV and when to use Rayleigh-CV or gap-CV variants. Principled construction of the candidate dictionary, beyond the hand-picked family used here, is another clear direction, potentially via data-dependent Mercer bases.

\emph{Applications and broader benchmarks.} Validation on molecular dynamics data (e.g., alanine dipeptide, $d\!=\!60$, $N\!=\!10^4$) would test the method where slow-mode discovery has direct scientific value. Financial, climate, and neural-recording data offer further targets. Head-to-head comparison with graph Laplacian eigenmaps would quantify the practical value of the dimension-free rate; comparison with neural spectral methods (SpectralNet, NeuralEF, Deep Ritz) would position kernel methods within the broader machine-learning landscape. Extensions to operator-valued KDM---learning Koopman or transfer-operator features rather than scalar eigenfunctions---would connect this work to the growing literature on kernel methods for dynamical systems~\citep{klus2020kernel,lee2024kernel}.

\bibliography{merged_kdm_20042026}
%bib_kdm_19042026, references_quasi_potential, references, refs_invariant, references_quasi_potential}
% ============================================================
% APPENDIX (reduced redundancy)
% ============================================================

\appendix

% ------------------------------------------------------------
\section{Additional theory and derivations}
\label{app:theory}

\subsection{Proof of Proposition~\ref{prop:pd-mixture} (positive definiteness of convex mixtures)}
\label{app:pd-mixture-proof}

\begin{proof}
Fix any $n\in\mathbb{N}$, points $(x_i)_{i=1}^n\subset\mathcal{X}$ and coefficients $c\in\mathbb{R}^n$.
Let $(K_\ell)_{ij}=k_\ell(x_i,x_j)$ and $K_\beta=\sum_{\ell=1}^L \beta_\ell K_\ell$.
Since each $K_\ell\succeq 0$ and $\beta_\ell\ge 0$,
\[
c^\top K_\beta c
=\sum_{\ell=1}^L \beta_\ell\, c^\top K_\ell c
\ge 0.
\]
Hence $K_\beta\succeq 0$ for all finite sets, so $k_\beta$ is positive definite.
\end{proof}

\subsection{Justification of Proposition~\ref{prop:rkhs-mixture-norm} (RKHS norm of a convex mixture)}
\label{app:rkhs-mixture-proof}

We give a standard construction that yields the infimal convolution identity stated in
Proposition~\ref{prop:rkhs-mixture-norm}.

Let $\beta\in\Delta_L$ and consider the Hilbert direct sum
\[
\mathcal{H}_\oplus
:=
\Hk_{k_1}\oplus\cdots\oplus \Hk_{k_L},
\qquad
\langle (f_\ell),(g_\ell)\rangle_{\oplus,\beta}
:=
\sum_{\ell=1}^L \frac{1}{\beta_\ell}\,\langle f_\ell,g_\ell\rangle_{\Hk_{k_\ell}}.
\]
Define the linear map $S:\mathcal{H}_\oplus\to \mathbb{R}^\mathcal{X}$ by
$S(f_1,\dots,f_L)=\sum_{\ell=1}^L f_\ell$.
Then the image $S(\mathcal{H}_\oplus)$ is a RKHS with kernel $k_\beta=\sum_{\ell=1}^L \beta_\ell k_\ell$,
and its norm is the quotient norm induced by $S$:
\[
\|f\|_{\Hk_{k_\beta}}^2
=
\inf\Big\{\|(f_\ell)\|_{\oplus,\beta}^2:\ f=\sum_{\ell=1}^L f_\ell,\ f_\ell\in\Hk_{k_\ell}\Big\}
=
\inf_{\substack{f=\sum_{\ell=1}^L f_\ell\\ f_\ell\in \Hk_{k_\ell}}}
\ \sum_{\ell=1}^L \frac{\|f_\ell\|^2_{\Hk_{k_\ell}}}{\beta_\ell}.
\]
This is exactly the identity in Proposition~\ref{prop:rkhs-mixture-norm}.

% ------------------------------------------------------------
\section{Nystr\"om discretization of KDM operators and generalized eigenproblem}
\label{app:spectral}

This appendix derives the finite-dimensional matrices used in the Nystr\"om inner loop,
i.e.\ the matrix surrogates in \eqref{eq:SigmaL_mats}--\eqref{eq:gen_eig_mat},
starting from the empirical RKHS operators in \eqref{eq:pv_emp_ops} and the generalized eigenproblem
\eqref{eq:pv_gen_eig_pop}.

\subsection{Landmark space and coordinates}
\label{app:landmark-subspace}

Fix a kernel $k$ (in particular $k=k_\beta$) with RKHS $\Hk_k$ and landmarks $Z=\{z_m\}_{m=1}^p$.
Define the landmark subspace
\[
\mathcal{H}_p := \mathrm{span}\{K_{z_m}(\cdot)=k(z_m,\cdot)\ :\ m=1,\dots,p\}\subset \Hk_k.
\]
Any $\psi\in\mathcal{H}_p$ can be written as
\begin{equation}
\psi(\cdot) = \sum_{m=1}^p a_m\, k(z_m,\cdot),
\qquad a\in\mathbb{R}^p.
\label{eq:app_psi-landmark}
\end{equation}

Define the landmark Gram matrix $W\in\mathbb{R}^{p\times p}$ and the cross-kernel matrix $C\in\mathbb{R}^{N\times p}$ by
\begin{equation}
W_{mn} = k(z_m,z_n),
\qquad
C_{im} = k(x_i,z_m).
\label{eq:app_CW_def}
\end{equation}
Then the vector of sample evaluations is $\psi(X)=Ca$.

\subsection{Covariance operator matrix}
\label{app:sigma-matrix}

Recall $\widehat{\Sigma}=\frac1N\sum_{i=1}^N K_{x_i}\otimes K_{x_i}$.
Let $e_m(\cdot)=K_{z_m}(\cdot)=k(z_m,\cdot)$.
By reproducing,
\[
\big\langle e_m,\widehat{\Sigma}\psi\big\rangle_{\Hk_k}
=
\frac1N\sum_{i=1}^N \langle e_m, K_{x_i}\rangle\, \langle K_{x_i},\psi\rangle
=
\frac1N\sum_{i=1}^N k(z_m,x_i)\,\psi(x_i).
\]
Writing $\psi(X)=Ca$ yields the coordinate form
\[
b := \big(\langle e_m,\widehat{\Sigma}\psi\rangle\big)_{m=1}^p
= \frac1N C^\top (Ca)
= \Big(\frac1N C^\top C\Big)a.
\]
Hence the landmark matrix surrogate is
\begin{equation}
\widehat{\Sigma}^{(p)} := \frac1N C^\top C \in\mathbb{R}^{p\times p}.
\label{eq:app_sigma_p}
\end{equation}
For $k=k_\beta$, this matches \eqref{eq:SigmaL_mats}.

\subsection{Dirichlet/gradient operator matrix}
\label{app:L-matrix}

Recall $\widehat{\mathcal{L}}=\frac1N\sum_{i=1}^N\sum_{j=1}^d \partial_j K_{x_i}\otimes \partial_j K_{x_i}$.
Define derivative cross-matrices $J^{(j)}\in\mathbb{R}^{N\times p}$ by
\begin{equation}
J^{(j)}_{im} = \partial_{x_j} k(x_i,z_m),
\qquad i=1,\dots,N,\ m=1,\dots,p,
\label{eq:app_Jj_def}
\end{equation}
and stack them as $J\in\mathbb{R}^{(Nd)\times p}$ so that $J^\top J=\sum_{j=1}^d (J^{(j)})^\top J^{(j)}$.

For $\psi(\cdot)=\sum_m a_m k(z_m,\cdot)$, we have $(\partial_j\psi)(x_i)=(J^{(j)}a)_i$.
Using reproducing properties for derivatives,
\[
\big\langle e_m,\widehat{\mathcal{L}}\psi\big\rangle_{\Hk_k}
=
\frac1N\sum_{i=1}^N\sum_{j=1}^d
\partial_{x_j}k(x_i,z_m)\,(\partial_j\psi)(x_i),
\]
so in vector form $b_L=\big(\langle e_m,\widehat{\mathcal{L}}\psi\rangle\big)_{m=1}^p=(\frac1N J^\top J)a$.
Thus
\begin{equation}
\widehat{L}^{(p)} := \frac1N J^\top J \in\mathbb{R}^{p\times p}.
\label{eq:app_L_p}
\end{equation}
For $k=k_\beta$, this matches \eqref{eq:SigmaL_mats}.

\subsection{Regularization and generalized eigenproblem}
\label{app:gen-eig}

KDM uses $\widehat{\mathcal{L}}_\lambda=\widehat{\mathcal{L}}+\lambda\Id$.
In landmark coordinates, $\Id$ contributes $W$ since
\[
\big\langle e_m, \Id\,\psi\big\rangle_{\Hk_k}
=
\langle e_m,\psi\rangle_{\Hk_k}
=
\sum_{n=1}^p W_{mn} a_n.
\]
Therefore,
\begin{equation}
\widehat{L}^{(p)}_\lambda := \widehat{L}^{(p)} + \lambda W.
\label{eq:app_Llam_p}
\end{equation}
Restricting $\widehat{\Sigma}\psi=\mu\,\widehat{\mathcal{L}}_\lambda\psi$ to $\mathcal{H}_p$ and testing against $\{e_m\}$ yields
\begin{equation}
\widehat{\Sigma}^{(p)} a = \mu\, \widehat{L}^{(p)}_\lambda a,
\qquad a\in\mathbb{R}^p,
\label{eq:app_gen_eig_mat}
\end{equation}
which is \eqref{eq:gen_eig_mat} for $k=k_\beta$.
Once $a$ is computed, sample evaluations follow from \eqref{eq:lifting_pv}.

\paragraph{Kernel mixtures.}
For $k_\beta=\sum_{\ell=1}^L\beta_\ell k_\ell$, the aggregated matrices $C_\beta,W_\beta,J_\beta$
are formed as in \eqref{eq:CWJ-beta}, with stabilization in \eqref{eq:W-stab},
and the identities above apply verbatim with $(C,W,J)=(C_\beta,W_\beta,J_\beta)$.

% ------------------------------------------------------------
\section{Kernel derivative identities (Gaussian kernels)}
\label{app:kernel-derivatives}

\subsection{Isotropic Gaussian}
\label{app:isotropic-gaussian}

Let $k(x,z)=\exp\!\big(-\|x-z\|^2/(2\sigma^2)\big)$ on $\mathbb{R}^d$.
Write $r=x-z$.
Then
\begin{align}
\nabla_x k(x,z) &= -\frac{r}{\sigma^2}\,k(x,z), \label{eq:app_grad_iso}\\
\partial_{x_j} k(x,z) &= -\frac{r_j}{\sigma^2}\,k(x,z), \\
\partial_{x_jx_j}^2 k(x,z) &= \Big(\frac{r_j^2}{\sigma^4}-\frac{1}{\sigma^2}\Big)\,k(x,z), \\
\Delta_x k(x,z) &= \Big(\frac{\|r\|^2}{\sigma^4}-\frac{d}{\sigma^2}\Big)\,k(x,z). \label{eq:app_lap_iso}
\end{align}

\subsection{Diagonal anisotropic Gaussian}
\label{app:anisotropic-gaussian}

Let $k(x,z)=\exp\!\big(-\tfrac12\sum_{j=1}^d r_j^2/\sigma_j^2\big)$ with $r=x-z$ and diagonal bandwidths $(\sigma_1,\dots,\sigma_d)$.
Then
\begin{align}
\partial_{x_j} k(x,z) &= -\frac{r_j}{\sigma_j^2}\,k(x,z), \label{eq:app_grad_aniso}\\
\partial_{x_jx_j}^2 k(x,z) &= \Big(\frac{r_j^2}{\sigma_j^4}-\frac{1}{\sigma_j^2}\Big)\,k(x,z), \\
\Delta_x k(x,z) &= \sum_{j=1}^d \Big(\frac{r_j^2}{\sigma_j^4}-\frac{1}{\sigma_j^2}\Big)\,k(x,z). \label{eq:app_lap_aniso}
\end{align}
For the $2$D parametrization used in the anisotropic dictionary,
$(\sigma_x,\sigma_y)=(\sigma r,\sigma/r)$, substitute $\sigma_1=\sigma_x$, $\sigma_2=\sigma_y$.

\subsection{Generator acting on kernel expansions (OU and 1D Langevin)}
\label{app:generator-on-kernels}

\paragraph{OU with $dX_t=-A X_t\,dt + \sqrt{2}\,dW_t$.}
The generator is $(\mathcal{G}f)(x)=-(Ax)\cdot\nabla f(x)+\Delta f(x)$.
Applying $\mathcal{G}$ to $x\mapsto k(x,z)$ uses \eqref{eq:app_grad_iso}--\eqref{eq:app_lap_aniso}:
\begin{equation}
(\mathcal{G}_x k)(x,z)
=
-(Ax)\cdot \nabla_x k(x,z) + \Delta_x k(x,z).
\label{eq:app_ou_generator_kernel}
\end{equation}
For $A=\mathrm{diag}(\alpha_1,\dots,\alpha_d)$ and anisotropic diagonal Gaussian, this becomes
\[
(\mathcal{G}_x k)(x,z)
=
\sum_{j=1}^d \Big(
\alpha_j x_j \frac{r_j}{\sigma_j^2}
+
\frac{r_j^2}{\sigma_j^4}-\frac{1}{\sigma_j^2}
\Big)\,k(x,z),
\qquad r_j=x_j-z_j.
\]

\paragraph{1D overdamped Langevin $dX_t=-V'(X_t)\,dt+\sqrt{2}\,dW_t$.}
The generator is $(\mathcal{G}f)(x)=-V'(x)f'(x)+f''(x)$.
For a 1D Gaussian kernel $k(x,z)=\exp(-(x-z)^2/(2\sigma^2))$,
\[
(\mathcal{G}_x k)(x,z)
=
- V'(x)\,\partial_x k(x,z) + \partial_{xx}^2 k(x,z),
\]
with $\partial_x k=-(x-z)\sigma^{-2}k$ and $\partial_{xx}^2 k=((x-z)^2\sigma^{-4}-\sigma^{-2})k$.

% ------------------------------------------------------------
\section{Rayleigh estimate used in the PDE residual loss}
\label{app:pde:rayleigh}

This appendix justifies the closed-form estimate of $\widehat{\lambda}$ used in \eqref{eq:pde-loss}.
Let $\phi\in\mathbb{R}^N$ be a learned eigenfunction evaluated on samples, and let $g:=G_h\phi\in\mathbb{R}^N$.
Consider the (empirical) least-squares fit of the eigen-relation $g \approx -\lambda \phi$:
\begin{equation}
\widehat{\lambda}
\in
\arg\min_{\lambda\in\mathbb{R}}
\frac1N\|g+\lambda \phi\|_2^2.
\label{eq:app_rayleigh_ls}
\end{equation}
Expanding and differentiating w.r.t.\ $\lambda$ yields the minimizer
\begin{equation}
\widehat{\lambda}
=
-\frac{\langle \phi, G_h\phi\rangle}{\langle \phi,\phi\rangle}.
\label{eq:app_rayleigh_closed}
\end{equation}

% ------------------------------------------------------------
\section{Eigenfunction gauge fixing and alignment for evaluation}
\label{app:anchoring}

\paragraph{Gauge fixing (recall).}
We enforce centering/orthonormality constraints in \eqref{eq:normalize-orth} (implemented by centering and a QR/Gram--Schmidt step under
$\langle u,v\rangle_N=\tfrac1N u^\top v$). This stabilizes the outer-loop differentiation through eigen-solves.

\subsection{Alignment when references are available}
\label{app:alignment}

When reference eigenfunctions are available, we align learned modes to references by maximizing absolute correlations via a Hungarian assignment and correcting signs by multiplying each learned mode by $\mathrm{sign}(\mathrm{corr})$.
This produces consistent plots and summary metrics such as AvgAbsCorr.

% ------------------------------------------------------------
\section{Reproducibility and implementation details}
\label{app:impl}

This appendix complements the experimental protocol in Section~\ref{subsec:exp-protocol} with a concise checklist.

\subsection{What we save per run}
For each seed and each regime, we save learned weights $\beta$, eigenfunction visualizations, per-mode correlation tables and AvgAbsCorr and
aggregated summaries across seeds (mean$\pm$std).

\subsection{Default numerical stabilizers}
Across all experiments we use:
\begin{itemize}
\item \textbf{Landmark jitter:} $W_\beta \leftarrow \tfrac12(W_\beta+W_\beta^\top)+\varepsilon_W I_p$ (cf.\ \eqref{eq:W-stab}).
\item \textbf{PSD flooring:} inverse square-roots of PSD matrices via eigendecomposition with eigenvalue floor $\varepsilon_{\mathrm{psd}}$.
\item \textbf{Mixture floor:} $\beta \leftarrow (1-\tau)\beta+\tau\mathbf{1}/L$ to prevent collapse.
\item \textbf{Gradient clipping:} optional clipping on $u$-gradients in the outer loop.
\end{itemize}

\end{document}